\documentclass[11pt]{article}
\usepackage[margin=1in]{geometry}
\usepackage[T1]{fontenc}

\usepackage[round,authoryear,semicolon]{natbib}
\usepackage[bottom]{footmisc}

\usepackage{algorithm}
\usepackage[compatible]{algpseudocode}

\usepackage[authstyle]{rcx}

\usepackage{enumitem}
\usepackage{booktabs}
\usepackage{diagbox}
\usepackage{multirow,makecell}
\usepackage{subcaption}
\newcommand{\LP}{\hyperlink{eq:opt}{\textnormal{LP}}}
\DeclareMathOperator{\err}{err}
\newcommand{\DPGap}{\Delta_\mathrm{DP}}
\newcommand{\Id}{\mathrm{Id}}
\newcommand{\dN}{d_\mathrm{N}}
\newcommand{\dP}{d_\mathrm{P}}
\newcommand{\dVC}{d_\mathrm{VC}}
\newcommand{\Nrv}{\mathrm{Nrv}}
\newcommand\indep{\protect\mathpalette{\protect\independenT}{\perp}}
\def\independenT#1#2{\mathrel{\rlap{$#1#2$}\mkern2mu{#1#2}}}

\newcommand*\centermathcell[1]{\omit\hfil$\displaystyle#1$\hfil\ignorespaces}

\usepackage{array}
\newcolumntype{C}[1]{>{\centering\arraybackslash}p{#1}}

 \makeatletter
 \newcommand{\phantomfill}[2]{{\mathpalette\mask@{{#1}{#2}}}}
 \newcommand{\mask@}[2]{\mask@@{#1}#2}
 \newcommand{\mask@@}[3]{%
   \settowidth{\dimen@}{$\m@th#1#2$}%
   \makebox[\dimen@]{$\m@th#1#3$}%
 }
 \makeatother

\title{Fair and Optimal Classification via Post-Processing}
\author{%
Ruicheng Xian\thanksAAffil{University of Illinois Urbana-Champaign. \texttt{\{\href{mailto:rxian2@illinois.edu}{rxian2},\href{mailto:langyin2@illinois.edu}{langyin2},\href{mailto:hanzhao@illinois.edu}{hanzhao}\}@illinois.edu}.}\and%
Lang Yin\footnotemarkAAffil[1]\and%
Han Zhao\footnotemarkAAffil[1]%
}
\date{}

\begin{document}
\maketitle

\begin{abstract}
To mitigate the bias exhibited by machine learning models, fairness criteria can be integrated into the training process to ensure fair treatment across all demographics, but it often comes at the expense of model performance.  Understanding such tradeoffs, therefore, underlies the design of fair algorithms.  To this end, this paper provides a complete characterization of the inherent tradeoff of \textit{demographic parity} on classification problems, under the most general multi-group, multi-class, and noisy setting.  Specifically, we show that the minimum error rate achievable by randomized and attribute-aware fair classifiers is given by the optimal value of a Wasserstein-barycenter problem.  On the practical side, our findings lead to a simple post-processing algorithm that derives fair classifiers from score functions, which yields the optimal fair classifier when the score is Bayes optimal.  We provide suboptimality analysis and sample complexity for our algorithm, and demonstrate its effectiveness on benchmark datasets.
\end{abstract}

\section{Introduction}
\setcounter{footnote}{1}

Machine learning models trained on biased data have been found to perpetuate and even amplify the bias against historically underrepresented and disadvantaged demographic groups at inference time~\citep{barocas2016BigDataDisparate,bolukbasi2016ManComputerProgrammer}. As a result, concerns of fairness have gained significant attention, especially as applications of these models expand to high-stakes domains such as criminal justice, healthcare, and finance~\citep{berk2021FairnessCriminalJustice}.  To mitigate the bias, a variety of fairness criteria and algorithms have been proposed~\citep{barocas2019FairnessMachineLearning,caton2020FairnessMachineLearning}, which impose mathematical or statistical constraints on the model to ensure equitable treatment under the respective fairness notions.  But these algorithms typically incur a cost to model performance as they improve model fairness~\citep{calders2009BuildingClassifiersIndependency,corbett-davies2017AlgorithmicDecisionMaking}.

It is not immediately clear whether the degradation in performance is attributed to artifacts of the algorithm, or possibly to the \textit{inherent tradeoff}---predictive power that must be given up for satisfying the criteria~\citep{hardt2016EqualityOpportunitySupervised,zhao2022InherentTradeoffsLearning}.  Hence the design of fair algorithms necessitates the understanding of this tradeoff, which would also provide insight to the implications of fairness in machine learning; yet, it remains an open problem for most fairness criteria and learning settings.

For the group fairness criterion of \textit{demographic parity}~(DP; \cref{def:dp}), a.k.a.~statistical parity, which requires statistical independence between model output and demographic group membership~\citep{calders2009BuildingClassifiersIndependency}, \citet{legouic2020ProjectionFairnessStatistical} and \citet{chzhen2020FairRegressionWasserstein} concurrently characterized the tradeoff between mean squared error~(MSE) and fairness on regression problems.  On classification problems, the inherent tradeoff in terms of error rate has only been studied under special cases: \citet{denis2022FairnessGuaranteeMulticlass} assumed binary groups, \citet{zeng2022BayesOptimalClassifiersGroup} and \citet{gaucher2022FairLearningWasserstein} assumed binary class labels, and \citet{zhao2022InherentTradeoffsLearning} assumed that the data distribution is \textit{noiseless}, i.e., the Bayes error rate is zero. We will close this gap and complete the characterization of the tradeoff of DP fairness in the most general classification setting.

\paragraph{Contributions.}  

This paper considers learning \textit{randomized} and \textit{attribute-aware} classifiers under (approximate) DP fairness in the general setting of multi-group, multi-class, and potentially \textit{noisy} data distributions.  We show that:

\begin{enumerate}
  \item The minimum classification error rate under DP is given by the optimal value of a (relaxed) Wasserstein-barycenter problem (\cref{sec:tradeoff}).

  \item This characterization reveals that the {optimal fair classifier}---one that satisfies DP while achieving the minimum error---is given by the composition of the Bayes optimal score function (minimum MSE regressor of the one-hot labels) and the optimal transports from the Wasserstein-barycenter problem (\cref{sec:post.proc.opt}).

  \item Based on the findings, we propose a post-processing method that derives fair classifiers from (pre-trained) score functions (\cref{sec:post.proc}).
Our method is instantiated for finite sample estimation in \cref{sec:algs} with sample complexity analysis.\footnote{\label{fn:code}Our code is available at \url{https://github.com/rxian/fair-classification}.}

\item Experiments on benchmark datasets demonstrate the effectiveness of our algorithm (\cref{sec:exp}), which achieves precise control of the tradeoff provided sufficient training data.
\end{enumerate}

\subsection{Related Work}\label{sec:related}

\begin{table}[t]
\caption{Characterizations of the inherent tradeoff of (strict) DP fairness.}
\label{tab:prior}
\centering
    \scalebox{0.9}{%
        \begin{tabular}{p{5.5cm}rlr}
        \toprule
        Problem Setting & \multicolumn{3}{c}{{Minimum Risk Under DP}}\\
        \midrule
        \multirowcell{4}[0ex][l]{Regression} & 
        \multirowcell{4}[0ex][r]{$\text{excess MSE}=$} & 
        \multirowcell{4}[0ex][l]{\hspace{-0.6em}$\displaystyle\smash{\min_{\phantomfill{q:\supp(q)\subseteq\{e_1,\cdots,e_k\}}{q:\supp(q)\subseteq\RR}}\sum_{a\in\calA}   w_a\,W_2^2(r^*_a,q)}$} & 
        \multirowcell{4}[0ex][l]{\refstepcounter{equation}(\theequation)\label{eq:rel.regression}} \\
        \\
        \\
        \\ \midrule
        \multirowcell{4}[0ex][l]{Classification (Noiseless Setting)} & 
        \multirowcell{4}[0ex][r]{$\text{excess}=\text{min.~error}=$} & 
        \multirowcell{4}[0ex][l]{\hspace{-0.6em}$\displaystyle\smash{\min_{q:\supp(q)\subseteq\{e_1,\cdots,e_k\}}\sum_{a\in\calA} \frac{w_a}2\enVert{p_a-q}_1}$} & 
        \multirowcell{4}[0ex][l]{\refstepcounter{equation}(\theequation)\label{eq:rel.cls.rel}} \\
        \\
        \\
        \\ \midrule
        \multirowcell{4}[0ex][l]{Classification (General Setting)} & 
        \multirowcell{4}[0ex][r]{$\text{minimum error}=$} & 
        \multirowcell{4}[0ex][l]{\hspace{-0.6em}$\displaystyle\smash{\min_{q:\supp(q)\subseteq\{e_1,\cdots,e_k\}}\sum_{a\in\calA}    \frac{w_a}2\,W_1(r_a^*,q)}$} & 
        \multirowcell{4}[0ex][l]{\refstepcounter{equation}(\theequation)\label{eq:rel.cls}}  \\
        \\
        \\
        \\
        \bottomrule
        \end{tabular}
    }
\end{table}

\paragraph{Inherent Tradeoff.}  The concept of barycenter appears in many analyses of the tradeoff of DP fairness.  Intuitively, by treating the barycenter---computed over the distributions of optimal model outputs (without constraints) on each group---as the output distribution that is required to be identical across groups under DP, the sum of distances to the barycenter is naturally related to the minimum fair error.

We review existing characterizations of the tradeoff of DP fairness below and draw connections to our result.  Denote the input by $X$, group membership by $A$, and target variable by $Y$ (for classification, the \textit{one-hot} label).  Let $r^*_a$ be the \textit{distribution} of the conditional mean on group $a$, $\E[Y\mid X, A=a]$, i.e., the minimum MSE estimates of $Y$ given $(X,A=a)$ (for classification, these are \textit{distributions} of class probabilities). Lastly, let $w_a\coloneqq\Pr(A=a)$ denote the proportion of each group. Then, under DP, on
\begin{itemize}
  \item regression problems~\citep{legouic2020ProjectionFairnessStatistical,chzhen2020FairRegressionWasserstein,chzhen2022MinimaxFrameworkQuantifying}, the minimum \textit{excess} risk in terms of MSE is given by the Wasserstein-$2$-barycenter (under the $\ell_2$ metric) over the $r^*_a$'s: \cref{eq:rel.regression};

\item
\textit{noiseless} classification problems~\citep{zhao2022InherentTradeoffsLearning}, the minimum/excess error rate is given by the TV-barycenter over the class priors, $p_a(e_i)\coloneqq\Pr(Y=e_i\mid A=a)$: \cref{eq:rel.cls.rel}, where $\frac12\|\cdot\|_1$ computes the total variation (TV);

\item classification problems in the general setting~(\cref{thm:tradeoffs}), the minimum error rate is given by the Wasserstein-$1$-barycenter (under the $\ell_1$ metric): \cref{eq:rel.cls}.
\end{itemize}
First, unlike regression, the support of the barycenters in \cref{eq:rel.cls.rel,eq:rel.cls} is restricted to $\{e_1,\cdots,e_k\}$, which represents the one-hot labels.  Combined with the fact that the error rate is the expected $\frac12\ell_1$ distance between the true class \textit{probabilities} and the output class \textit{assignments}, the minimum error rate equals the sum of $\frac12W_1$ distances to the barycenter under the $\ell_1$ metric.  Similarly, the use of the $W_2^2$ distance under $\ell_2$ in \cref{eq:rel.regression} reflects the MSE loss.  Second, our \cref{eq:rel.cls} recovers \cref{eq:rel.cls.rel} in the noiseless setting, because, under which, $r^*_a=p_a$ and $\frac12W_1=\frac12\|\cdot\|_1$.  \citet{denis2022FairnessGuaranteeMulticlass} and \citet{gaucher2022FairLearningWasserstein} also derived similar expressions for the tradeoff to ours, but only under binary group or class labels.

\paragraph{Post-Processing.}

Given a (biased) model, this family of mitigation algorithms post-process the model to satisfy fairness, e.g., via remapping the outputs~\citep{hardt2016EqualityOpportunitySupervised,pleiss2017FairnessCalibration}.  Existing algorithms for DP fairness include~\citep{fish2016ConfidenceBasedApproachBalancing,menon2018CostFairnessBinary,chzhen2019LeveragingLabeledUnlabeled,jiang2020WassersteinFairClassification,zeng2022BayesOptimalClassifiersGroup,denis2022FairnessGuaranteeMulticlass}, but they are limited to binary group and/or binary classification.

For multi-group and multi-class DP, the only applicable post-processing algorithm, to our knowledge, is due to~\citet{alghamdi2022AdultCOMPASFair}, which is based on model projection.  But the tradeoff of their algorithm is unclear as they did not directly relate error rate to the difference between the projected model and the original, and experiments show that their algorithm underperforms compared to ours, especially on tasks involving a large number of groups and classes.

\section{Preliminaries}

\paragraph{Notation.}
Denote the $(k-1)$-dimensional probability simplex by $\Delta_k\coloneqq \{z\in\RR^k_{\geq0} : \|z\|_1=1\}$, whose $k$ vertices are $\{e_1,\cdots,e_k\}$, where $e_i\in\RR^k$ is the vector of all zeros except for a single $1$ on the $i$-th coordinate.  Let $\calQ_k$ denote the collection of distributions supported on the vertices of $\Delta_k$. 
We will work with \textit{randomized functions} (\cref{def:rand.fn}), which have probabilistic outputs according to some distributions conditioned on the input.  Given a (randomized) function $f:\calX\rightarrow\calY$ and a distribution $p$ over $\calX$, we denote the \textit{push-forward} of $p$ by $f\sharp p$~(\cref{def:push.forward}).

\paragraph{Problem Setup.}

A $k$-class classification problem is defined by a joint distribution $\mu$ of input $X\in\calX$, demographic group membership (a.k.a.\ the sensitive attribute) $A\in\calA=[m]\coloneqq\{1,\cdots,m\}$, and class label in one-hot representation, $Y\in\calY=\{e_1,\cdots,e_k\}$; the class labels may be subject to noise originating from, e.g., the data collection process. Denote the marginal distribution of input $X$ by $\mu^X$, the conditional distribution of $\mu$ on group $A=a$ by $\mu_a$, and the group weight by $w_a\coloneqq\Pr_\mu(A=a)$.

The goal of fair classification is to find a \textit{randomized} and \textit{attribute-aware} classifier, $h:\calX\times \calA\rightarrow\calY$, that achieves the minimum classification error rate on $\mu$ subject to the constraints set by the designated fairness criteria.  Denote the component of $h$ associated with group $a$ by $h_a:\calX\rightarrow\calY$, i.e., $h_a(x)\equiv h(x,a)$. The error rate is defined as
\begin{align}
  \err(h)
  \coloneqq{}& \Pr(h_A(X)\neq Y) \\
  ={}&\sum_{a\in[m]}w_a\Pr(h_A(X)\neq Y\mid A=a)\\
  ={}&\sum_{a\in[m]}w_a\int_{\calX\times\calY}\Pr(h_a(x) \neq y)\dif \mu_a(x,y),\label{eq:err}
\end{align}
where the decomposition on the last line highlights the randomness of $h$.  For fairness, we consider the group criterion of demographic parity:

\begin{definition}[Approximate Demographic Parity]\label{def:dp}
For $\alpha\in[0,1]$, a classifier $h:\calX\times\calA\rightarrow\calY$ is said to satisfy $\alpha$-DP if $\DPGap(h) \leq \alpha$, which is defined as
\begin{align}
\DPGap(h)
&\coloneqq \max_{\substack{a,a'\in[m]\\y\in\calY}}  \envert*{ \Pr(h_A(X)=y \mid A=a) - \Pr(h_A(X)=y \mid A=a') } \\
&\phantom{\vcentcolon\mathrel{\mkern-1.2mu}}=
   \max_{a,a'\in[m]} \enVert*{h_a\sharp \mu^X_a- h_{a'}\sharp \mu^X_{a'} }_\infty,
\end{align}
where
\begin{align}
  \Pr(h_A(X)=y \mid A=a) =  \int_{\calX}\Pr(h_a(x) = y)\dif \mu_a^X(x)
\end{align}
is the proportion of outputs with class assignment $y$ on group $a$, and 
$\enVert*{p-q}_\infty\coloneqq\max_{z\in\calZ}|p(z)-q(z)|$ between two distributions $p,q$.
\end{definition}


We call a classifier \textit{$\alpha$-fair} if it satisfies $\alpha$-DP. The parameter $\alpha$ controls the tradeoff between fairness and (the maximum attainable) accuracy (due to the inherent tradeoff); setting $\alpha=0$ recovers the standard strict definition of DP.

Lastly, a (attribute-aware) score function $f:\calX\times\calA\rightarrow\Delta_k$ is a model that outputs probability vectors as estimates of the class probabilities, as in $f(x,a)_y\approx \Pr_{\mu_a}(Y=y\mid X=x)$.  A score function is said to be \textit{Bayes optimal}, denoted by $f^*$, if it computes the true class probabilities exactly,
\begin{align}
f^*_{a}(x)_i \coloneqq\Pr_{\mu_a}(Y=e_i\mid X=x)=\E_{\mu_a}[Y\mid X=x]_i;
\end{align}
it coincides with the minimum MSE estimator of the one-hot labels $Y$ given $(X,A=a)$.  We will often work with the quantity $r^*_{a}\coloneqq f^*_{a}\sharp \mu^X_a$, the distribution of true class probabilities conditioned on group $a$.

Given a (pre-trained) score function $f$, our post-processing method finds a (probabilistic) fair classifier by deriving from $f$. I.e., it returns classifiers of the form $(x,a)\mapsto g_a\circ f_a(x)$ for some post-processing maps $g_1,\cdots,g_m:\Delta_k\to\calY$.

\paragraph{Optimal Transport and Wasserstein Distance.}
Our analysis involves the concept of optimal transports and Wasserstein distance~\citep{villani2003TopicsOptimalTransportation}; the latter is a metric on the space of probability distributions.

\begin{definition}[Coupling] Let $p,q$ be probability distributions over $\calX$ and $\calY$, respectively.  A coupling $\gamma$ of $p,q$ is a joint distribution over $\calX\times\calY$ satisfying $p(x)=\int_{y\in\calY}\dif\gamma(x,y)$, $\forall x\in\calX$, and $q(y)=\int_{x\in\calX}\dif\gamma(x,y)$, $\forall y\in\calY$.  We denote the collection of couplings of $p,q$ by $\Gamma(p,q)$.
\end{definition}

\begin{definition}[Optimal Transport]\label{def:ot}
  Let $p,q$ be probability distributions over $\calX$ and $\calY$, respectively, and $c:\calX\times\calY\rightarrow[0,\infty)$ a cost function.
  The optimal transportation cost between $p$ and $q$ is given by 
  \begin{equation}
\inf_{\gamma\in\Gamma(p,q)}\int_{\calX\times\calY} c(x,y) \dif\gamma(x,y).    
  \end{equation}

  Let $\gamma^*$ be a minimizer, then the optimal transport from $p$ to $q$, denoted by $\calT^*_{p\rightarrow q,c}:\calX\rightarrow\calY$, is a (randomized) function satisfying $\gamma^*=(\Id\times \calT^*_{p\rightarrow q,c})\sharp p$, where $\Id$ is the identity map (that in the other direction is defined symmetrically).
\end{definition}

Intuitively, $\calT^*_{p\rightarrow q,c}$ specifies a plan for moving masses distributed according to $p$ to $q$ with the minimum total cost. In this plan, the mass located at each $x\in\calX$ is moved (probabilistically) to $\calT^*_{p\rightarrow q,c}(x)\in\calY$. 
The optimal transport can also be represented by the optimal coupling $\gamma^*\in\Gamma(p,q)$, as we can derive an optimal transport $\calT$ from $\gamma^*$ by setting $\Pr(\calT(X) = y \mid X=x) = \gamma^*(x,y)/\gamma^*(x, \calY)$, $\forall x,y$.\footnote{We will only consider transportation under the of $\ell_1$ cost of $(x,y)\mapsto \|x-y\|_1$, hence omit the dependency of $\calT^*_{p\rightarrow q}$ on $c$.}

Lastly, when $\calX=\calY$ is a metric space equipped with distance $d$, the optimal transportation cost between $p$ and $q$ under $c=d$ is equivalent to their Wasserstein-$1$ distance:

\begin{definition}[Wasserstein Distance]\label{def:wass}
  Let $p,q$ be probability distributions over a metric space $(\calX,d)$, and $r\in [1,\infty]$.  The Wasserstein-$r$ distance between $p$ and $q$ is 
\begin{equation}
W_r(p,q)=\rbr*{\inf_{\gamma\in\Gamma(p,q)}\int_{\calX\times\calX} d(x,x')^r \dif\gamma(x,x')}^{1/r}.
\end{equation}
\end{definition}

\section{Fair and Optimal Classification}\label{sec:wb}

In this section, we provide a characterization of the inherent tradeoff of DP fairness, then, based on the findings, propose and analyze a post-processing method for DP.

\subsection{Characterization of the Inherent Tradeoff}
\label{sec:tradeoff}

Our characterization comes from a reformulation of the classification problem assuming access to the Bayes optimal score.  On any generic (group-less) classification problem,

\begin{lemma}\label{lem:equiv} 
  Let $f^*:\calX\rightarrow\Delta_k$ be the Bayes optimal score function, define $r^*\coloneqq f^*\sharp\mu^X$, and fix $q\in\calQ_k$.  For any (randomized) classifier $h:\calX\rightarrow\calY$ satisfying $h\sharp\mu^X= q$, there exists a coupling $\gamma\in\Gamma(r^*,q)$ s.t. $\err(h) = \frac12\int_{\Delta_k\times\calY} \|s - y\|_1 \dif\gamma(s,y)$.  
Conversely, for any $\gamma\in\Gamma(r^*,q)$, there exists a randomized classifier $h$ satisfying $h\sharp\mu^X= q$ s.t.~the above equality holds.
\end{lemma}

\begin{figure}[t]
    \centering
    \includegraphics[width=0.25\linewidth]{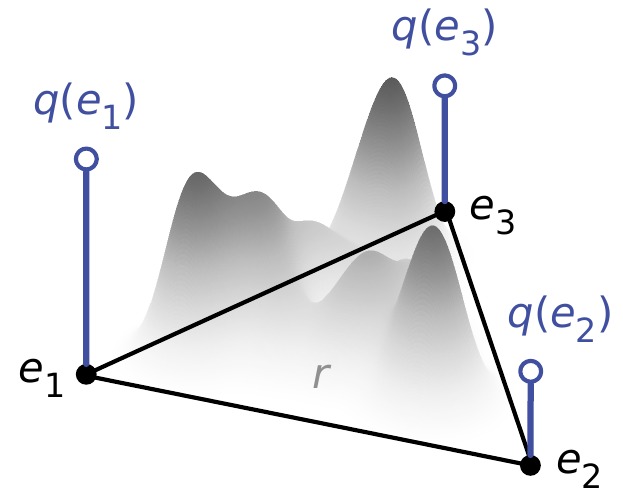}
    \caption{A distribution $r$ over the 2d simplex (density; grey surface), and a finite distribution $q\in\calQ_3$ over its vertices $\{e_1,e_2,e_3\}$ (blue spikes).}
    \label{fig:dist}
\end{figure}

It follows that minimizing classification error subject to having an output distribution of $q$ is equivalent to solving an optimal transport problem from $r^*$ to $q$ under the $\ell_1$ cost, because by \cref{def:wass},
$\min_{h:h\sharp \mu^X=q} \err(h) =\frac12 \min_{\gamma\in\Gamma(r^*,q)}\int \|s - y\|_1 \dif\gamma=\frac12 \,W_1(r^*,q)$.

Note that this reformulation allows for explicit control of the output distribution, which is well-suited for analyzing DP fairness since it only constrains the output distributions $q_a\coloneqq h_a\sharp \mu^X_a$; namely, that they need to be equal ($\alpha=0$) or close: $\DPGap(h)\leq \alpha\iff \max_{a,a'}\|q_a-q_{a'}\|_\infty\leq \alpha$ by \cref{def:dp}.
Therefore, for attribute-aware classifiers, whose component $h_a$'s can be optimized independently, the discussions above immediately give the following characterization of the minimum error rate under DP:

\begin{theorem}[Minimum Fair Error Rate]\label{thm:tradeoffs}  
  Let $\alpha\in[0,1]$, $f^*:\calX\times\calA\rightarrow\Delta_k$ be the Bayes optimal score function, and define $r^*_a\coloneqq f^*_a\sharp\mu^X_a$, $\forall a\in[m]$. With $W_1$ under the $\ell_1$ metric,
\begin{align}
\err^*_\alpha \coloneqq{}& \min_{h: \DPGap(h)\leq\alpha} \err( h) =\min_{\substack{q_1,\cdots,q_m\in\calQ_k\\\max_{a,a'}\|q_a-q_{a'}\|_\infty\leq\alpha}} \sum_{a\in[m]} \frac{w_a}2 \, W_1(r^*_a,q_a). \label{eq:barycenter} 
\end{align}
\end{theorem}

It could be viewed as a (relaxed) \textit{Wasserstein-barycenter} problem on the $r^*_a$'s under the special case where the support of the barycenter(s) $q_a$ is restricted to the vertices $\{e_1,\cdots,e_k\}$.  It is a convex problem (in the primal form presented above), and can be simplified under certain assumptions: if the problem is noiseless (and $\alpha=0$), it reduces to the TV-barycenter problem in \cref{eq:rel.cls.rel}~(\cref{prop:tv}); this result is first established in~\citep{zhao2022InherentTradeoffsLearning}, but only under $m=k=2$.

Under strict DP fairness ($\alpha=0$), the inherent tradeoff, namely the excess risk incurred by the DP constraint, is
\begin{equation}
 \min_{q\in\calQ_k} \sum_{a\in[m]}\frac{w_a}2\,W_1(r^*_a,q)-\sum_{a\in[m]}\min_{q_a\in\calQ_k} \frac{w_a}2\, W_1(r^*_a,q_a)\geq0;
\end{equation}
the second term is the Bayes error rate, achieved by the classifier $(x,a)\mapsto e_{\argmax_i f^*_a(x)_i}$.
The tradeoff is expected to be large on problems with very different $r^*_a$'s, and equals to zero when they are identical (i.e., $\E_\mu[Y\mid X,A]\indep A$; since all groups would have the same optimal decision rule), meaning that enforcing DP would not degrade model performance.  But we point out that the tradeoff could be zero even if $\E_\mu[Y\mid X,A]\not\indep A$, partly due to the nonuniqueness of the optimal classifier (\cref{exp:1}).

Lastly, \citet{zhao2022InherentTradeoffsLearning} concluded that in the noiseless setting, the tradeoff is zero if and only if the class priors are identical, i.e., $\E_\mu[Y\mid A]\indep A$.  But this condition is no longer sufficient in the general setting (\cref{exp:2}).

\subsection{Optimal Fair Classifier via Post-Processing}\label{sec:post.proc.opt}

In addition to characterizing the minimum error rate under DP, we also show that the optimal fair classifier can be obtained by deriving from the Bayes optimal score $f^*$:

\begin{theorem}[Optimal Fair Classifier]\label{thm:post.proc} 
  Let $\alpha\in[0,1]$, $f^*:\calX\times\calA\rightarrow\Delta_k$ be the Bayes optimal score function, $(q^*_1,\cdots,q^*_m)$ a minimizer of \cref{eq:barycenter}, and $\calT^*_{r^*_a\rightarrow q^*_a}$  the optimal transport from $r^*_a$ to $q^*_a$ under the $\ell_1$ cost, $\forall a\in[m]$. We have
\begin{equation}
  (x,a) \mapsto \calT^*_{r^*_a\rightarrow q^*_a}\circ f^*_a(x) \in \argmin_{h: \DPGap(h)\leq\alpha} \err( h).
\end{equation}
\end{theorem}

This result is a consequence of the construction used in the proof of \cref{lem:equiv} (deferred to \cref{sec:proof.3}), and reveals the form of the optimal fair classifier as a composition of $f^*$ and optimal transports from $r^*_a$'s to the minimizing $q^*_a$'s of \cref{eq:barycenter} (see \cref{fig:dist} for a picture of these distributions).  It immediately suggests a three-step method for learning optimal fair classifiers: \textit{(i)}~learn the Bayes optimal score function $f^*$, e.g., via minimizing MSE w.r.t.~the one-hot label $Y$,
\begin{equation}\label{eq:sq.loss}
f^*=\argmin_{f:\calX\times\calA\rightarrow\Delta_k} \E_\mu\sbr*{\|f(X,A)-Y\|_2^2},
\end{equation}
\textit{(ii)}~find a minimizer $(q^*_1,\cdots,q^*_m)$ of the barycenter problem in \cref{eq:barycenter}, \textit{(iii)}~compute the optimal transports $\calT^*_{r^*_a\rightarrow q_a^*}$, and finally return $(x,a) \mapsto  \calT^*_{r^*_a\rightarrow q_a^*}\circ f_a^*(x)$.  The last two steps (reproduced in \cref{alg:post}) post-process $f^*$.

\subsection{Post-Processing Any Score Function}\label{sec:post.proc}

In practice, however, the Bayes optimal $f^*$ may not be exactly learned due to computational cost, or difficulties in representation, optimization, and generalization~\citep{woodworth2017LearningNonDiscriminatoryPredictors}.  Instead, we will often work with suboptimal scores $f \approx f^*$ (e.g., pre-trained by a vendor).  This section analyzes the applicability and suboptimality of \cref{alg:post} for post-processing non-Bayes optimal score functions.

\begin{algorithm}[tb]
   \caption{Post-Process for $\alpha$-DP}
   \label{alg:post}
\begin{algorithmic}[1]
   \STATE {\bfseries Input:} $\alpha\in[0,1]$, score function $f:\calX\times\calA\rightarrow\Delta_k$, marginal distribution $\mu^{X,A}$ of $(X,A)$
   \STATE Define $w_a\coloneqq \Pr_\mu(A=a)$ and $r_a\coloneqq f_a\sharp \mu^X_a$, $\forall a\in[m]$ 
   \STATE $(q_1,\cdots,q_m)\gets$ minimizer of Eq.~(\ref{eq:barycenter})
   \FOR{$a=1$ {\bfseries to} $m$}
    \STATE $\calT^*_{r_a\rightarrow q_a}\gets$ optimal transport from $r_a$ to $q_a$ under $\ell_1$ cost
\ENDFOR
\STATE{\bfseries Return:} $(x,a) \mapsto \calT^*_{r_a\rightarrow q_a}\circ f_a(x)$
\end{algorithmic}
\end{algorithm}

Given an arbitrary score function $f:\calX\times\calA\rightarrow\Delta_k$, we want to find post-processing maps $g_a:\Delta_k\rightarrow\calY$ such that the derived classifier $(x,a)\mapsto g_a\circ f_a(x)$ satisfies DP fairness, and ideally, achieves the minimum error rate among all fair classifiers derived from $f$.  

Let $\bar h(x,a)\coloneqq  \calT^*_{r_a\rightarrow q_a}\circ f_a(x)$ denote the  classifier obtained from applying \cref{alg:post} to $f$. First, $\bar h$ is always $\alpha$-fair regardless of $f$, because $\bar h_a\sharp \mu^X_a =  q_a$, $\forall a\in[m]$ by construction, and $\DPGap(\bar h)=\max_{a,a'}\|q_a-q_{a'}\|_\infty\leq\alpha$ from the constraints in \cref{eq:barycenter}. 
Second, the suboptimality of $\bar h$ can be upper bounded by the $L^1$ difference between $f$ and the Bayes optimal score $f^*$:

\begin{theorem}[Error Propagation]\label{prop:error.propa}
Let $\alpha\in[0,1]$, $f:\calX\times\calA\rightarrow\Delta_k$ be a score function, and $f^*$ the Bayes optimal score function.  For the $\alpha$-fair classifier $\bar h$ obtained from applying \cref{alg:post} to $f$,
\begin{equation}
0\leq  \err(\bar h) - \err^*_\alpha \leq \E \sbr*{ \enVert{ f(X,A) - f^*(X,A) }_1},
\end{equation}
where $\err^*_\alpha$ is defined in \cref{eq:barycenter}.
\end{theorem}

Hence, whereas the degradation in performance of the post-processed $\bar h$ from \cref{alg:post} if $f=f^*$ is attributed entirely to the inherent tradeoff (\cref{thm:post.proc}), it is not the case when $f\neq f^*$ due to the loss of information about $Y$.  We may, however, guarantee that $\bar h$ is optimal among all fair classifiers derived from $f$ if it satisfies \textit{group-wise distribution calibration}~\citep{kull2015NovelDecompositionsProper}, i.e., the output predictions correctly convey the class probabilities:

\begin{definition}\label{def:calibration}
  A score function $f:\calX\times\calA\rightarrow\Delta_k$ is said to be group-wise distribution calibrated if 
  $\Pr_\mu(Y=e_i \mid f(X,a)=s, A=a) = s_i$, $\forall s\in\Delta_k$, $i\in[k]$, $a\in[m]$.
\end{definition}

If $f$ is not calibrated, but labeled data is available, one could learn mappings $u_a:\Delta_k\rightarrow\Delta_k$ and compose it with $f$ to recalibrate it. 
The optimal calibration maps, $u^*_a$ (in the sense that they achieve calibration without incurring further information loss), are by definition the minimum MSE estimators of $Y$ given $(f_a(X),A=a)$.

To see why the optimality of $\bar h$ among all fair classifiers derived from $f$ is guaranteed provided calibration, note that finding the optimal fair \textit{post-processing map} for $f$ is equivalent to finding the optimal fair \textit{classifier} on a new problem $\mu'$ derived from the original $\mu$ under an input transformation, given by the joint distribution of $(X'\coloneqq f_A(X), A, Y)$.  Also, the Bayes optimal score on $\mu'$ coincide with the optimal calibration map, as $\E_{\mu'}[Y\mid X'=s, A=a]=u_a^*(s)$.  So by \cref{thm:post.proc}, \cref{alg:post} finds post-processing map $g_a$'s s.t.\ $(x',a)\mapsto g_a\circ u_a^*(x')$ is the optimal fair classifier on $\mu'$, whereby $(x,a)\mapsto g_a\circ (u_a^*\circ f_a)(x)$ is optimal among all derived fair classifiers (the term in the parentheses is the recalibrated score).  Finally, we remark that the suboptimality of $\bar h$ due to miscalibration can be bounded using \cref{prop:error.propa} by simply replacing the reference $f^*$ with the calibrated score $(x,a)\mapsto u^*_a\circ f_a$.

\section{Finite Sample Estimation}\label{sec:algs}

We have discussed post-processing for DP assuming access to the distribution $\mu^{X,A}$. In this section, we instantiate our \cref{alg:post} for post-processing using finite samples:

\begin{assumption}\label{ass:samples}
We have $n$ i.i.d.\ samples of $(X,A)$ that are independent of the score function $f$ being post-processed.
\end{assumption}

Denote the samples from group $a$ by $(x_{a,i})_{i\in[n_a]}$, and their number by $n_a$.  Define $\hat w_a\coloneqq n_a/n$, and the empirical distribution $\hat r_a\coloneqq \frac1{n_a}\sum_{i\in[n_a]}\delta_{f_a(x_{a,i})}$, where $\delta$ is the Dirac delta function.

We also analyze the sample complexity for both the fairness and the error rate of the returned classifier. For simplicity, we assume $f$ to be calibrated; otherwise, \cref{prop:error.propa} can be used to bound the suboptimality due to miscalibration.

\begin{assumption}\label{ass:calibration}
  The score function $f$ being post-processed is \textit{group-wise calibrated} (\cref{def:calibration}).
\end{assumption} 

Let $r_a\coloneqq f_a\sharp \mu^X_a$, $\forall a\in[m]$. Recall from \cref{sec:post.proc} that the error rate of the optimal derived $\alpha$-fair classifier is
  \begin{equation}
  \err^*_{\alpha,f} \coloneqq  \min_{\substack{q_1,\cdots,q_m\in\calQ_k\\\max_{a,a'}\|q_a-q_{a'}\|_\infty\leq\alpha}} \sum_{a\in[m]} \frac {w_a}2\,W_1(r_a,q_a).
\end{equation}

This section is divided into three subsections w.r.t.\ the continuity of the distributions of the score, $r_1,\cdots,r_m$.

\subsection{The Finite Case}\label{sec:discrete}

We start with the case where the $r_a$'s have finite supports, i.e., $|\calR_a|<\infty$ where $\calR_a\coloneqq \supp(r_a)$.  Note that this does not mean that the input distribution $\mu^X$ is finite.

If the true probability mass of the $r_a$'s were known, then \cref{alg:post} can be implemented by a linear program:
\begin{alignat}{3}
    &\hypertarget{eq:opt}{\textnormal{LP}}:\quad
       & \!\min_{\substack{q_1,\cdots,q_m\geq0\\\gamma_1,\cdots,\gamma_m\geq0}}\;\;\, & \sum_{\mathclap{a\in[m]}}\;\, \sum_{s\in\calR_a, y\in\calY} \frac{w_a}{2}\,\|s-y\|_1 \, \gamma_a(s,y)   \\
    && \centermathcell{\text{s.t.}\;\;\,} 
    & \sum_{\mathclap{s'\in\calR_a}} \;\gamma_a(s',y) = q_a(y), && \forall a\in[m],\, y\in\calY, \\
    &&& \sum_{\mathclap{y'\in\calY}} \; \gamma_a(s,y') = r_a(s), && \forall a\in[m],\, s\in\calR_a,\\
    &&& \envert*{q_a(y) - q_{a'}(y)} \leq \alpha, && \forall a,a'\in[m],\, y\in\calY,
\end{alignat}
where $q_a\in\Delta_k$ and $\gamma_a\in \RR^{|\calR_a|\times k}$. This program simultaneously finds a minimizer $(q^*_1,\cdots,q^*_m)$ of the barycenter problem in \cref{eq:barycenter} and the optimal transports $\calT^*_{r_a\rightarrow q_a^*}$ (used in \cref{thm:post.proc}) in the form of couplings $(\gamma^*_1,\cdots,\gamma_m^*)$: namely, each $\calT^*_{r_a\rightarrow q^*}$ is a randomized function satisfying $\Pr(\calT^*_{r_a\rightarrow q^*}(R) = y \mid R=s) =\gamma_a^*(s,y)/\sum_{y'\in\calY}\gamma_a^*(s, y')$, for all $s\in\calR_a$.

If the true pmfs of the $r_a$'s are unknown but finite samples as in \cref{ass:samples} are given, we proceed with solving \LP\ defined on the empirical $\hat w_a$ and $\hat r_a$'s, which will give us estimated $\hat q_a$'s and $\calT^*_{\hat r_a\rightarrow\hat q_a}$'s.  Then, we post-process $f$ via $\hat h(x,a) = \calT^*_{\hat r_a\rightarrow \hat q_a}\circ f_a(x)$. The sample complexity is:

\begin{theorem}[Sample Complexity, Finite Case]\label{thm:finite.sample.discrete} 
  Let $\alpha\in[0,1]$, $f:\calX\times\calA\rightarrow\Delta_k$ be a score function, and assume $|\calR_a|\coloneqq\supp(f_a\sharp\mu_a^X)<\infty$, $\forall a\in[m]$.  W.p.~at least $1-\delta$ over the random draw of samples in \cref{ass:samples}, for the classifier $\hat h$ derived above, and $n\geq \Omega(\max_a\ln(m/\delta)/w_a)$,
  \begin{equation}
    \DPGap(\hat h)\leq \alpha + O\rbr*{\max_a\sqrt{\frac{|\calR_a|\ln\rbr{{m}/{\delta}}}{nw_a}}  };
\end{equation}
in addition, with \cref{ass:calibration},
\begin{equation}
  \err(\hat h) - \err^*_{\alpha,f} \leq O\rbr*{\max_a\sqrt{\frac{|\calR_a|\ln\rbr{{m}/{\delta}}}{nw_a}}  }.
\end{equation}
\end{theorem}

\subsection{The Continuous Case}\label{sec:continuous}

When the $r_a$'s are continuous,\footnote{\label{fn:continuous}I.e., the probability measure does not give mass to sets whose intersection with $\Delta_k$ has Hausdorff dimension less than $k-1$.} given finite samples, we may still solve \LP\ defined on $\hat w_a$ and $\hat r_a$'s to estimate the optimal output distribution $\hat q_a$'s under $\alpha$-DP, but the empirical transports $\calT^*_{\hat r_a\rightarrow\hat q_a}$ are no longer usable for post-processing in this case, since by continuity, the inputs to the transports at inference time will be unseen almost surely (i.e., $f_a(x)\notin f_a[(x_{a,i})_{i\in[n_a]}]$ a.s.\ for $x\sim \mu^X_a$), or in other words, they cannot extrapolate to the full support of $r_a$.  So after obtaining the $\hat q_a$'s, we will need to estimate the optimal transports $\calT^*_{r_a\rightarrow\hat q_a}$ from (the population) $r_a$'s to the $\hat q_a$'s.

\begin{figure}[t]
    \centering
    \includegraphics[width=1\linewidth]{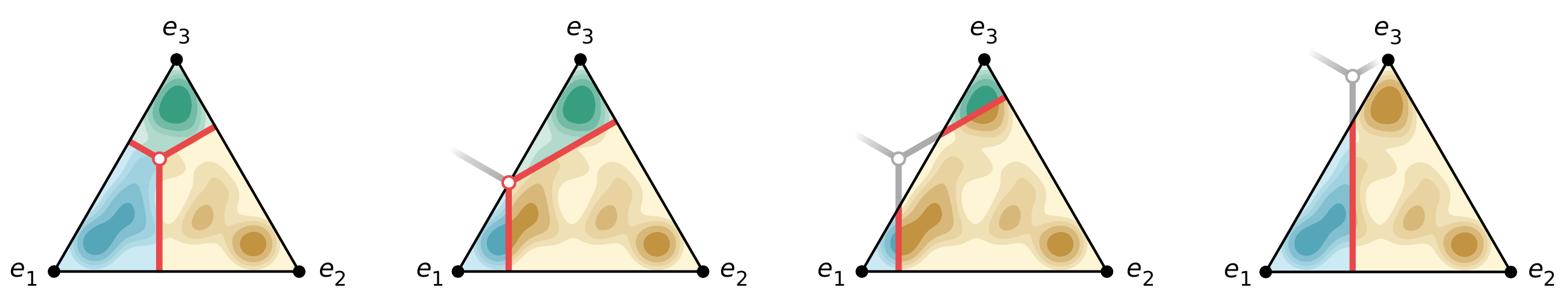}
    \caption{Examples of simplex-vertex optimal transports ($k=3$; with different vertex distributions).  All points in the lower-left blue partition are transported to $e_1$, lower-right yellow to $e_2$, and upper green to $e_3$.  The transports are described by a Y-shaped boundary.}
    \label{fig:nn}
\end{figure}

Since $\supp(\hat q_a)=\{e_1,\cdots,e_k\}$ is finite, this makes finding $\calT^*_{r_a\rightarrow\hat q_a}$ a semi-discrete optimal transport problem~\citep{genevay2016StochasticOptimizationLargescale,staib2017ParallelStreamingWasserstein,chen2019GradualSemiDiscreteApproach}, for which, a common procedure in existing work is to reformulate optimal transport as a convex optimization problem over a vector $\psi_a\in\RR^k$ using the Kantorovich-Rubinstein dual and the $c$-transform of the Kantorovich potential $\phi_a$: concretely, for each $a\in[m]$,
\begin{align}
  W_1(r_a, \hat q_a) &= \inf_{\gamma_a\in\Gamma(r_a,\hat q_a)}\int_{\Delta_k\times\calY} \|s-y\|_1 \dif\gamma_a(s,y) \\
  &= \sup_{\substack{\phi_a:\Delta_k\rightarrow\RR,\, \psi_a\in\RR^k\\\phi_a(s)+\psi_{a,i}\leq \|s-e_i\|_1}}\rbr*{ \EE_{S\sim r_a} [\phi_a(S)] + \sum_{i=1}^k \psi_{a,i} \hat q_a(e_i) }\\
  &= \sup_{\psi_a\in\RR^k} \rbr*{ \EE_{S\sim r_a}\sbr*{\min_i(\|S-e_i\|_1-\psi_{a,i})} + \sum_{i=1}^k \psi_{a,i} \hat q_a(e_i)},\label{eq:erm}
\end{align}
and $\psi^*_a$ can be optimized, e.g., using (stochastic) gradient ascent. 
Moreover, \citet{gangbo1996GeometryOptimalTransportation} showed that in the semi-discrete case, the optimal transport $\calT^*_{r_a\rightarrow\hat q_a}$ belongs to the parameterized function class 
\begin{equation}
\calG_k \coloneqq  \cbr*{s\mapsto e_{\argmin_{i\in[k]}(\|s - e_i\|_1 -\psi_i)} : \psi\in\RR^k}  
\end{equation}
(break ties to the tied $e_i$ with the largest index $i$)
with $\psi^*_a$ as its parameter. See \cref{fig:nn} for pictures of semi-discrete optimal transports when $k=3$.  Compared to the empirical transports $\calT^*_{\hat r_a\rightarrow\hat q_a}$, the domain of $\calT^*_{r_a\rightarrow\hat q_a}\in\calG_k$ covers the support of $r_a$, i.e., it can handle future unseen inputs.

Note that \citet{gangbo1996GeometryOptimalTransportation} assumed the cost function to be strictly convex and superlinear (Assumptions~H1 to H3 in their paper), which are not satisfied by our $\ell_1$ cost~\citep{ambrosio2003ExistenceStabilityResults}. Hence, as a technical contribution, we provide a proof in \cref{sec:transport} for the existence and uniqueness of the optimal transport in $\calG_k$ on simplex-vertex transportation problems under the $\ell_1$ cost, via analyzing its geometry.

\begin{figure}[t]
    \centering
    \includegraphics[width=\linewidth]{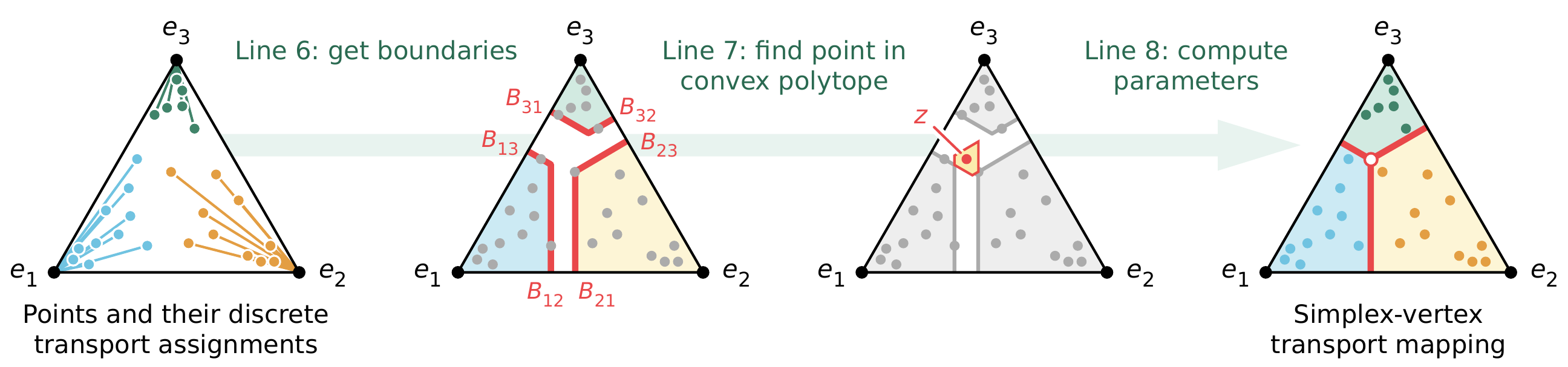}
    \caption{Extract a simplex-vertex transport $\calT\in\calG_3$ that agrees with the discrete optimal transport (\cref{lst:ln.6,lst:ln.7,lst:ln.8,lst:ln.9} of \cref{alg:cont}). For illustrative purposes, the discrete transport on the left does not split mass; otherwise, there would be disagreements on the boundaries.}
    \label{fig:transport}
\end{figure}

\paragraph{Our Implementation.}

To recap, for post-processing in the finite sample and continuous case, we \textit{(i)}~get the $\hat q_a$'s from solving \LP, then \textit{(ii)}~estimate the $\calT^*_{r_a\rightarrow\hat q_a}$'s; as discussed above, the typical approach for solving this semi-discrete transportation problem is via optimizing $\psi_a$ w.r.t.\ \cref{eq:erm} (taking expectation over the empirical $\hat r_a$).

However, instead of estimating from scratch, by leveraging the observation that each of the empirical transports $\calT^*_{\hat r_a\rightarrow\hat q_a}$ obtained (as byproducts) from solving \LP\ already achieves the optimal value of $W_1(\hat r_a,\hat q_a)$ of \cref{eq:erm}, and their decision boundaries is simply described by $\binom{k}{2}=\Theta(k^2)$ hyperplanes (recall \cref{fig:nn}), we can directly extract a set of transport mappings $\calT_a\in\calG_k$ from $\calT^*_{\hat r_a\rightarrow\hat q_a}$. 
The procedure for this is provided on \cref{lst:ln.6,lst:ln.7,lst:ln.8,lst:ln.9} in \cref{alg:cont} (with step-by-step illustration in \cref{fig:transport}; formal derivations are deferred to \cref{sec:transport}), which amounts to finding a feasible point in a polytope, formulated as a linear program (\cref{sec:find.point}).  Each of the extracted $\calT_a$ will agree with $\calT^*_{\hat r_a\rightarrow\hat q}$ on all points in $\hat r_a$ except for those that lie on the $\Theta(k^2)$ boundaries, and by continuity of $r_a$, the number of disagreements between them is $O(k^2)$ almost surely.

\begin{algorithm}[t]
   \caption{Post-Process for $\alpha$-DP (Finite Samples, Continuous Case)}
   \label{alg:cont}
\begin{algorithmic}[1]
   \STATE {\bfseries Input:} $\alpha\in[0,1]$, score function $f:\calX\times\calA\rightarrow\Delta_k$, samples $((x_{a,i})_{i\in[n_a]})_{a\in[m]}$
   \STATE Define $\hat w_a\coloneqq \frac{n_a}n$ and $\hat r_a \coloneqq \frac1{n_a}\sum_{i}\delta_{f_a(x_{a,i})}$, $\forall a\in[m]$ 
   \STATE $(\gamma_1,\cdots,\gamma_m) \gets \text{minimizer of } \LP$ on $\hat w_a$ and $\hat r_a$'s 
    \STATE Define $v_{ij}\coloneqq e_j-e_i$
   \FOR{$a=1$ {\bfseries to} $m$}
\STATE $B_{a,ij}\gets \{0\}\cup \max\{ f_a(x_{a,\ell})^\T v_{ij} + 1 : \ell \text{ s.t. }  \gamma_a(f_a(x_{a,\ell}),e_i) > 0 \}$ \label{lst:ln.6}
\STATE $z_a\gets \text{point in } \bigcap_{i\neq j} \{ x\in \RR^k : x^\T v_{ij} \geq B_{a,ij} -1 \}$ \label{lst:ln.7}
\STATE $\psi_{a,i}\gets 2z_a^\T v_{i1}$, $\forall i\in[k]$ \label{lst:ln.8}
\STATE $\calT_a\gets (s\mapsto e_{\argmin_{i}(\|s - e_i\|_1 -\psi_{a,i})})$  \label{lst:ln.9}
\ENDFOR
\STATE{\bfseries Return:} $(x,a) \mapsto \calT_a\circ f_a(x)$
\end{algorithmic}
\end{algorithm}

Our implementation in \cref{alg:cont} involves $(m+1)$ linear programs, where \LP\ dominates with $O(nk)$ variables and constraints, and takes, e.g., $\widetilde O(\mathrm{poly}(nk))$ time to solve to (near-)optimality using interior point methods~\citep{vaidya1989SpeedingupLinearProgramming}. Its sample complexity is:

\begin{theorem}[Sample Complexity, Continuous Case]\label{thm:finite.sample.continuous}
 Let $\alpha\in[0,1]$, $f:\calX\times\calA\rightarrow\Delta_k$ be a score function, and assume that $f_a\sharp\mu_a^X$ is continuous, $\forall a\in[m]$.  W.p.~at least $1-\delta$ over the random draw of samples in \cref{ass:samples}, for the classifier $\hat h$ obtained from applying \cref{alg:cont} to $f$, and $n\geq \Omega(\max_a\ln(m/\delta)/w_a)$,
\begin{equation}
\DPGap(\hat h) \leq \alpha+ O\rbr*{\max_a\rbr*{ \sqrt{\frac{k+\ln\rbr{{mk}/{\delta}}}{nw_a}} + \frac k{nw_a}}};
\end{equation}
in addition, with \cref{ass:calibration},
\begin{equation}
  \err(\hat h) - \err^*_{\alpha,f} \leq O\rbr*{ \max_a \rbr*{\sqrt{\frac{k\ln (m/\delta)}{nw_a}}  + \frac{k^2}{nw_a} }}.
\end{equation}
\end{theorem}

The first term in both expressions is the sample complexity of PAC learning (with the complexity of $\calG_k$ analyzed in \cref{thm:natdim}), and the second term comes from the disagreement between $\calT_a$ and $\calT^*_{\hat r_a\rightarrow\hat q}$ on $\hat r_a$ discussed above.

\subsection{The General Case}\label{sec:mixed}

For completeness, we briefly discuss the general case where the $r_a$'s are neither finite nor continuous (i.e., contain atoms), which is handled by smoothing the $r_a$'s using an i.i.d.\ noise generator $\rho$ with a continuous distribution.\footref{fn:continuous}

The smoothing is done by perturbing (samples from) $r_a$ with random noise drawn from $\rho$, i.e.,  $\tilde r_a \coloneqq  u_\rho\sharp r_a$, where $u_\rho$ is a randomized function s.t.~$u_\rho(s)\sim s+N$ with $N\sim\rho$; it is not hard to show that the resulting $\tilde r_a$ is continuous. Now that $u_\rho\circ f_a$ produces continuous score distributions, we may apply the same algorithm in \cref{sec:continuous} for DP post-processing. The resulting classifier, $\bar h_\rho(x,a)\coloneqq \calT^*_{\tilde r_a\rightarrow q_a}\circ u_\rho\circ f_a(x)$, is $\alpha$-fair regardless of the choice of $\rho$, and the suboptimality incurred by the smoothing procedure is controlled by the \textit{bandwidth} of $\rho$:

\begin{theorem}[Error Propagation, Smoothing] \label{thm:error.smoothing} 
  Let $\alpha\in[0,1]$, $f:\calX\times\calA\rightarrow\Delta_k$ be a score function, and $\rho$ a continuous distribution with finite first moment.  Under \cref{ass:calibration}, for the $\alpha$-fair classifier $\bar h_{\rho}$ derived above, 
\begin{equation}
  0\leq\err(\bar h_{\rho}) - \err^*_{\alpha,f}\leq\E_{N\sim\rho} [\|N\|_1].
\end{equation}
\end{theorem}

E.g., if $\rho=\text{Laplace}(0,b\, I_k)$, then $\E[\|N\|_1] = kb$, and the suboptimality due to smoothing is less than $kb$; in practice, the smallest-allowable $b$ depends on machine precision.

\section{Experiments} \label{sec:exp}

Our proposed DP post-processing \cref{alg:cont} is evaluated on four benchmark datasets: the UCI~Adult dataset for income prediction~\citep{kohavi1996ScalingAccuracyNaiveBayes}, the ACSIncome dataset~\citep{ding2021RetiringAdultNew}---an extension of the Adult with much more examples (1.6 million vs.\ 48,842)---on which we consider a binary setting where the sensitive attribute is gender and the target is whether the income is over \$50k, as well as a multi-group multi-class setting with five race categories and five income buckets; the Communities \& Crime dataset~\citep{redmond2002DatadrivenSoftwareTool}, and the BiasBios dataset~\citep{de-arteaga2019BiasBiosCase}, where the task is to predict occupations from biographies.  We highlight the effectiveness of our algorithm by comparing it to \texttt{FairProjection}~\citep{alghamdi2022AdultCOMPASFair}, which is to our knowledge the only other fair post-processing algorithm for multi-group and multi-class classification.

On each dataset, we split the data into pre-training, post-processing, and testing. We first train a linear logistic regression scoring model on the pre-training split, then perform DP post-processing.  On BiasBios, the model is trained on embeddings of biographies computed by a pre-trained BERT language model from the \texttt{bert-base-uncased} checkpoint~\citep{devlin2019BERTPretrainingDeep}.  Additional details, including hyperparameters, are in \cref{sec:exp.additional}.

\begin{figure}[t]
    \centering
    \includegraphics[height=.345\linewidth]{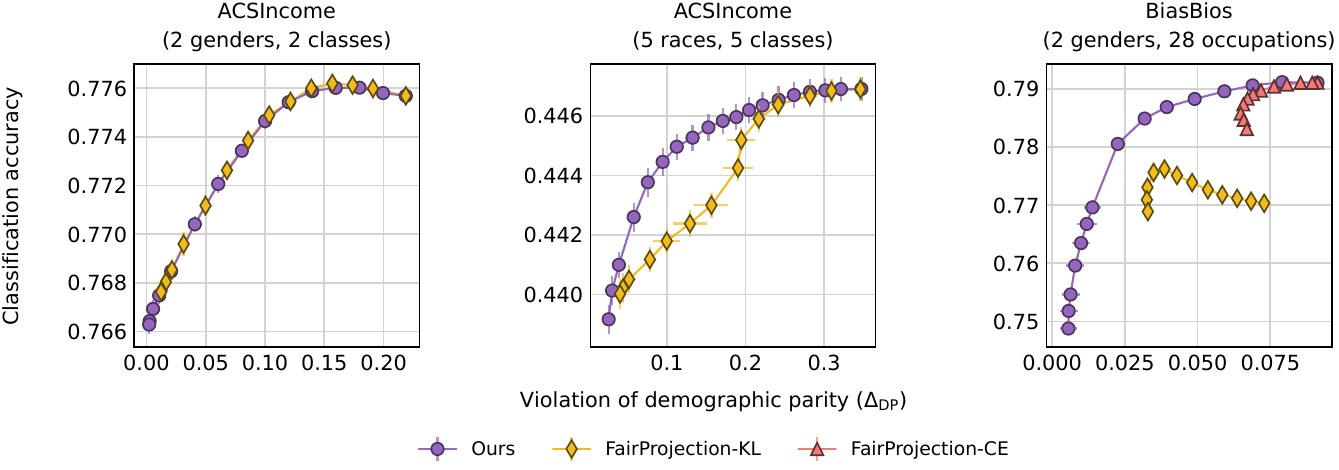}
  \caption{Tradeoff curves between accuracy and $\DPGap$ (\cref{def:dp}).  Scoring model is logistic regression. Error bars indicate the standard deviation over 10 runs with different random splits. Running time is reported in appendix \cref{tab:runtime}.  On ACSIncome, \texttt{FairProjection-KL} and \texttt{FairProjection-CE} have similar results.}
    \label{fig:exp}
\end{figure}

\paragraph{Results.}

The results on ACSIncome and BiasBios are shown in \cref{fig:exp} (those on Adult and Communities are deferred to appendix \cref{fig:exp.2}).  Across all tasks, our method is effective at reducing the disparity, and almost precise control of $\DPGap$ via $\alpha$ is achieved on tasks with sufficient data.
Compared to \texttt{FairProjection}, our method can achieve lower $\DPGap$ and produces better tradeoff curves, and its advantage is most evident under multi-class settings (e.g., BiasBios).

While our method achieves a significant degree of DP fairness, there remains a gap to $\DPGap=0$ in our results, especially on tasks with more groups and classes (e.g., ACSIncome under the multi-group multi-class setting).  This could be due to a potential violation of the continuity assumption required by \cref{alg:cont} (\cref{sec:continuous}), but we suspect the main reason to be insufficient data.  Recall from \cref{thm:finite.sample.continuous} that the sample complexity scales as $\widetilde O(\sqrt{k/nw_a})$ in the worst-case $a$, where $w_a$ is the proportion of group $a$. This means that good generalization performance hinges on collecting adequate amounts of data from minority groups.  Lastly, as discussed in \cref{sec:post.proc}---although not empirically explored here---higher accuracies could be achieved if the scores were calibrated prior to post-processing.

\section{Further Related Work}\label{sec:more.related}

\paragraph{Fairness Criteria.}
This paper focused on the \textit{group} criterion of demographic parity, which is defined on population-level statistics~\citep{verma2018FairnessDefinitionsExplained,kearns2018PreventingFairnessGerrymandering}. Other group criteria include parity of true positive and/or negative rates~\citep{hardt2016EqualityOpportunitySupervised}, predictive rates~\citep{chouldechova2017FairPredictionDisparate,berk2021FairnessCriminalJustice,zeng2022FairBayesOptimalClassifiers}, accuracy~\citep{buolamwini2018GenderShadesIntersectional}, etc.  Besides group-level, there are criteria defined on the \textit{individual-level}~\citep{dwork2012FairnessAwareness,sharifi-malvajerdi2019AverageIndividualFairness}, which require the model to output similar predictions to individuals deemed to be similar under application and context-specific measures designed by the practitioner.

\paragraph{Mitigation Methods.} In addition to post-processing, there are data pre-processing methods~\citep{calmon2017OptimizedPreProcessingDiscrimination}, as well as in-processing ones via constrained optimization~\citep{kamishima2012FairnessAwareClassifierPrejudice,zafar2017FairnessConstraintsMechanisms,agarwal2019FairRegressionQuantitative} or fair representation learning~\citep{zemel2013LearningFairRepresentations,madras2018LearningAdversariallyFair,zhao2019ConditionalLearningFair,song2019LearningControllableFair}.  There are methods under other learning settings and paradigms, such as unsupervised learning~\citep{chierichetti2017FairClusteringFairlets,backurs2019ScalableFairClustering,li2020DeepFairClustering}, ranking~\citep{zehlike2017FAIRFair}, and sequential decision making~\citep{joseph2016FairnessLearningClassic,joseph2017FairAlgorithmsInfinite,gillen2018OnlineLearningUnknown,chi2022ReturnParityMarkov}.

\section{Conclusion}

In this paper, we characterized the inherent tradeoff of DP fairness on classification problems in the most general setting, and proposed an effective post-processing method with suboptimality and sample complexity analyses.  Our implementation uses linear program solvers; while they enjoy stability and consistent performance, a potential concern is scalability to larger numbers of classes or samples.  It would be of practical value to analyze an implementation that uses more time efficient optimization methods, e.g., gradient descent~\citep{staib2017ParallelStreamingWasserstein}.

Technically, we studied the geometry of the optimal transport between distributions supported on the simplex and its vertices. A main result is that when the distributions are semi-discrete, the optimal transport is unique, and is given by the $c$-transform of the Kantorovich potential, including under the $\ell_1$ cost which is neither strictly convex nor superlinear.  This result may be of independent theoretical interest to the community.

Our results add to the line of work that study the inherent tradeoffs of fairness criteria, which we believe would benefit practitioners in the design of fair machine learning systems, and contribute to a better understanding of the implications of fairness in machine learning.

\section*{Acknowledgements}
The authors thank Jane Du, Yuzheng Hu, and Seiyun Shin for feedback on the draft.  The work of HZ was supported in part by the Defense Advanced Research Projects Agency (DARPA) under Cooperative Agreement Number: HR00112320012, a Facebook Research Award, and Amazon AWS Cloud Credit.

\bibliographystyle{plainnat-eprint}
\bibliography{references}

\newpage
\appendix

\section{Additional Discussions in Section~\ref{sec:tradeoff}}\label{sec:wwb}

This section discusses the reduction of \cref{thm:tradeoffs} to TV-barycenter in the noiseless setting, and the examples in \cref{sec:tradeoff}.

\subsection{Reduction to TV-Barycenter}\label{sec:tvb}

When the classification problem $\mu$ is noiseless, i.e., the (unconstrained) Bayes error rate is zero, $\min_h \err(h)=0$, we show that \cref{thm:tradeoffs} reduces to a relaxed TV-barycenter problem.  For strict DP ($\alpha=0$), this result is previously established by \citet{zhao2022InherentTradeoffsLearning} for the binary case of $m=k=2$.  Our reduction here holds for the general case.

Note that noiselessness means that there exist a deterministic labeling function $y:\calX\times\calA\rightarrow\calY$ s.t.~$Y=y(X,A)$ (almost surely).  Therefore, the Bayes optimal score $f^*=y$, and we have
\begin{equation}\label{eq:tv.equiv}
  r^*_a = p_a \coloneqq f^*_a\sharp\mu^X_a \quad\text{where}\quad p_a(e_i)= \Pr_\mu(Y=e_i\mid A=a).
\end{equation}

\begin{theorem}[Minimum Fair Error Rate, Noiseless Setting]\label{prop:tv}
Let $\alpha\in[0,1]$, and suppose $\mu$ is noiseless. We have
  \begin{equation}
   \min_{h: \DPGap(h)\leq\alpha} \err(h) = \min_{\substack{q_1,\cdots,q_m\in\calQ_k\\\max_{a,a'}\|q_a-q_{a'}\|_\infty\leq\alpha}} \sum_{a\in[m]} \frac{w_a}2 \,W_1(r^*_a,q_a)=\min_{\substack{q_1,\cdots,q_m\in\calQ_k\\\max_{a,a'}\|q_a-q_{a'}\|_\infty\leq\alpha}}\sum_{a\in[m]}\frac{w_a}2\|p_a-q_a\|_1.
    \end{equation}
\end{theorem}

This is due to $\supp(p_a)\subseteq\{e_1,\cdots,e_k\}$ sharing the same finite support with any $q_a\in\calQ_k$, whereby $\frac12 \,W_1(r^*_a,q_a)=\frac12 \,W_1(p_a,q_a)=\frac12\|p_a-q_a\|_1$.  Specifically, recall the fact that $W_1$ under the 0-1 distance is equal to $\|\cdot\|_1$:

\begin{proposition}\label{prop:tv.w1.01}
  Let $p,q$ be probability measures on $\calX$ with metric $d(x,y)=\1[x\neq y]$, where $\1[\cdot]$ denotes the indicator function. We have $W_1(p,q)=\frac12\|p-q\|_1$.
\end{proposition}

\begin{proof}
 By definition, under the metric $d(x,y)=\1[x\neq y]$,
  \begin{align}
    W_1(p,q)&= \inf_{\gamma\in\Gamma(p,q)} \int_{\calX\times\calX} \1[x\neq y] \dif \gamma(x,y) \\
    &= \rbr*{1-\sup_{\gamma\in\Gamma(p,q)} \int_{\calX\times\calX} \1[x= y] \dif \gamma(x,y)} \\
    &= 1- \int_\calX \min\rbr*{p(x), q(x)} \dif x\\
    &= \int_\calX \rbr*{p(x)- \min\rbr*{p(x), q(x)} }\dif x\\
    &= {\int_\calX \max\rbr*{0, p(x) - q(x)} \dif x} \\
    &= \frac 12{\int_\calX  \envert*{ p(x) - q(x) } \dif x},
\end{align}
where line 3 is due to $\gamma(x,x)\leq \min(p(x),q(x))$ for all $\gamma\in\Gamma(p,q)$, plus there always exists a coupling s.t.~$\gamma(x,x)= \min(p(x),q(x))$, and line 5 to $\int_\calX { p(x) - q(x) } \dif x = 0$.
\end{proof}

\begin{proof}[Proof of \cref{prop:tv}]
Because $\supp(r^*_a)=\supp(p_a)\subseteq\{e_1,\cdots,e_k\}$, the $\ell_1$ distance (in $W_1$) between any $s\in \supp(r^*_a)$ and $y\in\{e_1,\cdots,e_k\}$ simplifies to $\|s-y\|_1 = 2 \1[s\neq y]$. The result then follows from \cref{prop:tv.w1.01}.
\end{proof}

Moreover, in this case, we have closed-form solution for the optimal fair classifier in \cref{thm:post.proc}:

\begin{theorem}[Optimal Fair Classifier, Noiseless Setting]
  Let $\alpha\in[0,1]$, suppose $\mu$ is noiseless, let $y:\calX\times\calA\rightarrow\calY$ be the ground-truth labeling function, and $(q^*_1,\cdots,q^*_m)$ a minimizer of \cref{eq:barycenter}.
  Define
 \begin{equation}
    d_a(e_i)\coloneqq\frac{\max(0,q_a^*(e_i)-p_a(e_i))}{\sum_{j\in[k]}\max(0,q_a^*(e_j)-p_a(e_j))},\quad s_a(e_i)\coloneqq\frac{\max(0,p_a(e_i)-q_a^*(e_i))}{p_a(e_i)},\quad\forall a\in[m],\,i\in[k],
  \end{equation}
  and randomized functions $\calT_a:\calY\rightarrow\calY$ for each $a\in[m]$ satisfying
  \begin{equation}
    \calT_a(e_i) = \begin{cases}
    e_i  & \text{w.p. } s_a(e_i) d_a(e_i) + (1-s_a(e_i)),\\
    e_j & \text{w.p. } s_a(e_i) d_a(e_j), \quad\qquad\qquad\qquad \forall j\in[m],\,j\neq i.
    \end{cases}
  \end{equation}
  We have
\begin{equation}
  (x,a) \mapsto \calT_{a}\circ y_a(x) \in \argmin_{h: \DPGap(h)\leq\alpha} \err( h).
\end{equation}
\end{theorem}

\begin{proof}

We first verify that $(\calT_a\circ y_a) \sharp \mu^X_a = q^*_a$, which would imply $\DPGap(h)\leq\alpha$ by the constraints in \cref{eq:barycenter}. For all $e_i\in\calY$,
\begin{align}
  \Pr(\calT_a\circ y_a(X)=e_i\mid A=a) &= p_a(e_i)(s_a(e_i)d_a(e_i)+(1-s_a(e_i))) + d_a(e_i)\sum_{i\neq j}  p_a(e_j)s_a(e_j)\\
  &=p_a(e_i)(1-s_a(e_i)) + d_a(e_i)\sum_{j\in[k]}  p_a(e_j)s_a(e_j) \\
  &=p_a(e_i) - \max(0,p_a(e_i)-q_a^*(e_i)) + \max(0,q_a^*(e_i)-p_a(e_i)) \\
  &=q_a^*(e_i),
\end{align}
where line 3 is because $\sum_{i}p_a(e_i)-q_a^*(e_i) = 0\implies \sum_{i} \max(0,p_a(e_i)-q_a^*(e_i)) = \sum_{i}\max(0,q_a^*(e_i)-p_a(e_i))$, and the last line is from case analysis.

Next, we compute the error rate on group $a\in[m]$. By construction, its accuracy conditioned on $Y=y_a(X)=e_i$ is 
\begin{align}
  \Pr(\calT_a\circ y_a(X)=e_i \mid A=a, Y=e_i) = \begin{cases}
1- s_a(e_i) & \text{if $d_a(e_i)=0$},\\
1=1- s_a(e_i) & \text{if $d_a(e_i)> 0\iff s_a(e_i) = 0$},\\
\end{cases}
\end{align}
so the error rate is
\begin{align}
    \Pr(\calT_a\circ y_a(X)\neq Y \mid A=a) &= \sum_{i\in [k]} p_a(e_i)\, (1-\Pr(\calT_a\circ y_a(X)=e_i \mid A=a, Y=e_i)) \\
&=\sum_{i\in[k]} p_a(e_i) s_a(e_i) \\
&= \sum_{i\in[k]}\max(0,p_a(e_i)-q_a^*(e_i)) \\
&= \frac12\sum_{i\in[k]}|p_a(e_i)-q_a^*(e_i)|.
\end{align}
We conclude by combining the error on all groups and invoking \cref{prop:tv}.
\end{proof}

\subsection{Examples Regarding the Tradeoff of DP Fairness}\label{sec:examples}

In the remarks of \cref{thm:tradeoffs}, we discussed properties of the inherent tradeoff of error rate for DP fairness, which we illustrated here with two concrete examples.

It is discussed that the tradeoff could be zero even when the distribution of class probabilities $r^*_a\coloneqq f^*_a\sharp\mu^X_a$ differ, or equivalently, $\E_\mu[Y\mid X,A]\indep A$. This means that on certain problem instances, the Bayes error rate is simultaneously achieved by an unfair classifier and a fair one; in other words, the cost of DP fairness is zero.  Such cases arise from the nonuniqueness of the optimal classifier.  They would not occur on regression problems (with MSE), where the optimal regressor is always unique (namely, $f^*_a(x)=\E_\mu[Y\mid X,A=a]$).  

\begin{example}\label{exp:1}
  Consider the two-group binary classification problem given by
\begin{gather}
  \Pr_{\mu_1}(Y=e_1\mid X=x)=1 \quad\text{and}\\
  \Pr_{\mu_2}(Y=e_1\mid X=x)=\Pr_{\mu_2}(Y=e_2\mid X=x)=\frac12 \quad \text{for all $x\in\calX$}.
\end{gather}
The optimal classifier on group 1 is the constant function $x\mapsto e_1$, and all classifiers on group 2 yield the same (hence optimal) error rate of $\frac12$, including $x\mapsto e_1$ which when combined with the optimal group 1 classifier achieves DP and the (group-balanced) Bayes error rate of $\frac14$. 
\end{example}

In~\citep{zhao2022InherentTradeoffsLearning}, it is concluded that in the noiseless setting, the inherent tradeoff is zero if and only if the class prior distributions are the same, $p_a(e_i)= \Pr_\mu(Y=e_i\mid A=a)$ and $r^*_a = p_a \coloneqq f^*_a\sharp\mu^X_a$ in this case, or equivalently $\E_\mu[Y\mid A]\indep A$.  However, for the general case, this is no longer sufficient for the tradeoff to be zero (a sufficient condition here is $\E_\mu[Y\mid X,A]\indep A$).  

\begin{example}\label{exp:2}
    Consider the two-group binary classification problem of two inputs, $\calX=\{1,2\}$, given by
\begin{gather}
\begin{aligned}
  \Pr_{\mu_1}(Y=e_1\mid X=1)&=1,&\Pr_{\mu_1}(Y=e_1\mid X=2)&=0, \\
  \Pr_{\mu_1}(Y=e_2\mid X=1)&=0,&\Pr_{\mu_1}(Y=e_2\mid X=2)&=1,&&\text{with}\\
  \Pr_{\mu_1}(X=1)&=\frac13, &\Pr_{\mu_1}(X=2)&=\frac23,&&\text{and}
  \end{aligned}
\\
  \Pr_{\mu_2}(Y=e_1\mid X=x)=\frac 13 \quad \text{for all $x\in\{1,2\}$}.
\end{gather}
Note that the class prior on both groups is $(\frac13,\frac23)$.  The unique optimal classifier on group 1 is $x\mapsto e_x$, and the unique optimal classifier on group 2 is the constant $x\mapsto e_2$, but this combination do not satisfy DP, since the output distribution on group 1 is $(\frac13,\frac23)$ but that on group 2 is $(0,1)$.  Since all other classifiers including the fair ones have strictly higher error rates, the tradeoff is nonzero.
\end{example}

\section{Proofs for Section~\ref{sec:wb}}
\label{sec:proof.3}

To make our arguments rigorous, we begin by providing a definition of randomized functions via the Markov kernel.  These definitions will be frequently referred to in the proofs in this section, and that of \cref{thm:error.smoothing}.

\begin{definition}[Markov Kernel]\label{def:markov.kernel}
  A Markov kernel from a measurable space $(\calX,\calS)$ to $(\calY,\calT)$ is a mapping $\calK:\calX\times \calT\rightarrow[0,1]$, such that $\calK(\cdot,T)$ is $\calS$-measurable $\forall T\in\calT$, and $\calK(x,\cdot)$ is a probability measure on $(\calY,\calT)$ $\forall x\in\calX$.
\end{definition}

\begin{definition}[Randomized Function]\label{def:rand.fn}
  A randomized function $f:(\calX,\calS)\rightarrow(\calY,\calT)$ is associated with a Markov kernel $\calK:\calX\times \calT \rightarrow[0,1]$, and for all $x\in\calX, T\in\calT$, $\Pr(f(x)\in T) = \calK(x,T)$.
\end{definition}

\begin{definition}[Push-Forward by Randomized Function]\label{def:push.forward}
    Let $p$ be a measure on $(\calX,\calS)$ and $f:(\calX,\calS)\rightarrow(\calY,\calT)$ a randomized function with Markov kernel $\calK$. The push-forward of $p$ under $f$, denoted by $f\sharp p$, is a measure on $\calY$ given by $f\sharp  p(T) =  \int_{\calX} \calK(x,T)  \dif p(x)$ for all $T\in\calT$.
\end{definition}

Also, let blackboard bold $\1$ denote the indicator function, where $\1[P]=1$ if the predicate $P$ is true, else $0$.

We provide the proofs to \cref{lem:equiv,thm:tradeoffs,thm:post.proc,prop:error.propa}.
The proofs to these results and the ones in \cref{sec:algs} all make use of the following rewriting of the error rate as an integral over a coupling:

\begin{lemma}\label{lem:rewrite}
  Let $f^*:\calX\rightarrow\Delta_k$ be the Bayes optimal score function. The error rate of any randomized classifier $h:\calX\rightarrow\calY$ can be written as
\begin{equation}
  \err( h)  = \frac12\int_{\Delta_k\times\calY} \|s - y\|_1 \Pr( f^*(X)=s, h(X) = y ) \dif(s,y)= \frac12\E \sbr* { \enVert{  f^*(X) - h(X) }_1}.
\end{equation}
\end{lemma}

Note that the joint distribution $\Pr$ of $(f^*(X), h(X))$ is a coupling belonging to $\Gamma(f^*\sharp\mu^X, h\sharp \mu^X)$.

\begin{proof} The accuracy of $h$ is
  \begin{align}
     1-\err( h) &= 1- \Pr(h(X) \neq Y) = \Pr( h(X) = Y) \\
    &= \int_{\Delta_k}\sum_{i\in[k]} \Pr(Y=e_i, h(X) = e_i,  f^*(X)=s)  \dif s\\
    &= \int_{\Delta_k}\sum_{i\in[k]} \Pr(Y=e_i, h(X) = e_i \mid f^*(X)=s) \Pr_\mu(f^*(X)=s)  \dif s\\
    &= \int_{\Delta_k}\sum_{i\in[k]} \Pr_\mu(Y=e_i \mid f^*(X)=s) \Pr( h(X) = e_i \mid f^*(X)=s ) \Pr_\mu(f^*(X)=s)  \dif s\\
    &= \int_{\Delta_k} \sum_{i\in[k]} s_i \Pr(f^*(X)=s,  h(X) = e_i ) \dif s,
\end{align}
where line 4 follows from $X\indep Y$ given $f^*(X)$, since $f^*(X)=\E_\mu[Y\mid X]$ fully specifies the pmf of $Y$ conditioned on $X$.  Next, because $\Pr( f^*(X) = \cdot, h(X)=\cdot )$ is a probability measure,
\begin{align}    
    \err( h) &= \int_{\Delta_k}\sum_{i\in[k]} (1-s_i)\Pr( f^*(X)=s, h(X) = e_i ) \dif s \\
    &= \frac12 \int_{\Delta_k}\sum_{i\in[k]} \|s - e_i\|_1\Pr( f^*(X)=s, h(X) = e_i ) \dif s \\    
    &\equiv \frac12\int_{\Delta_k\times\calY} \|s - y\|_1 \Pr( f^*(X)=s, h(X) = y ) \dif(s,y)\\
    &\equiv \frac12\E \sbr* { \enVert{  f^*(X) - h(X)}_1},
\end{align}
where the second equality is due to an identity stated in \cref{eq:simplex.dist}.
\end{proof}

\begin{lemma}[Full Version of \cref{lem:equiv}] \label{lem:equiv.full}
  Let $f^*:\calX\rightarrow\Delta_k$ be the Bayes optimal score function, define $r^*\coloneqq f^*\sharp\mu^X$, and fix $q\in\calQ_k$.   
  For any randomized classifier $h:\calX\rightarrow\calY$ with Markov kernel $\calK$ satisfying $h\sharp\mu^X= q$, the coupling $\gamma\in\Gamma(r^*,q)$ given by 
\begin{equation}
\gamma(s,y) = \int_{{f^*}^{-1}(s)} \calK(x,y)\dif \mu^X(x),
\end{equation}
where ${f^*}^{-1}(s)\coloneqq\{x\in\calX : f^*(x)=s\}$, satisfies
\begin{equation}\label{eq:equiv.full}
\err(h)  =  \frac12\int_{\Delta_k\times\calY} \|s - y\|_1 \dif\gamma(s,y).
\end{equation}
Conversely, for any $\gamma\in\Gamma(r^*,q)$, the randomized classifier $h$ with Markov kernel 
\begin{equation}
\calK(x,T)=\gamma(f^*(x),T)/\gamma(f^*(x),\calY)  
\end{equation}
satisfies $h\sharp\mu^X= q$ and \cref{eq:equiv.full}.
\end{lemma}

\begin{proof}
  We begin with the first direction. Let a randomized classifier $h$ with Markov kernel $\calK$ satisfying $h\sharp\mu^X= q$ be given.  We verify that the coupling constructed above belongs to $\Gamma(r^*,q)$:
  \begin{align}
    \int_{\calY}\gamma(s,y)\dif y 
    &= \int_{\calY} \int_{{f^*}^{-1}(s)} \calK(x,y)\dif \mu^X(x)\dif y \\
    &= \int_{{f^*}^{-1}(s)} \int_{\calY} \calK(x,y)\dif y\dif \mu^X(x)\\
    &= \int_{ {f^*}^{-1}(s)} \dif \mu^X(x) \\
    &= \Pr_{\mu^X}(f^*(X)=s) = r^*(s),
  \end{align}
  where line 3 follows from \cref{def:markov.kernel} of Markov kernels, and line 5 from the definition of push-forward measures;
  \begin{align}
    \int_{\Delta_k}\gamma(s,y)\dif s 
    &= \int_{\Delta_k}\int_{ {f^*}^{-1}(s)} \calK(x,y)\dif \mu^X(x)\dif s \\
    &=  \int_{\calX} \calK(x,y)\dif \mu^X(x) \\
    &= \int_{\calX} \Pr(h(X)=y\mid X=x)\dif \mu^X(x) \\
    &= \Pr(h(X)=y) = q(y),
  \end{align}
where line 3 follows from \cref{def:rand.fn} of randomized function, and line 5 is by assumption.
  
Next, by \cref{lem:rewrite} and the same arguments above,
  \begin{align}
    \err(h) 
    &= \frac12\int_{\Delta_k\times\calY} \|s - y\|_1 \Pr( f^*(X)=s, h(X) = y ) \dif (s,y) \\
    &= \frac12\int_{\Delta_k\times\calY}  \|s - y\|_1 \rbr*{\int_{x} \Pr( f^*(X)=s, h(X) = y, X=x) \dif x} \dif (s,y) \\
    &= \frac12\int_{\Delta_k\times\calY}  \|s - y\|_1 \rbr*{\int_{{f^*}^{-1}(s)} \Pr( h(X) = y, X=x) \dif x} \dif (s,y) \\
    &= \frac12\int_{\Delta_k\times\calY}  \|s - y\|_1 \rbr*{\int_{ {f^*}^{-1}(s)} \Pr( h(X) = y\mid X=x) \dif \mu^X(x)} \dif (s,y) \\
    &= \frac12\int_{\Delta_k\times\calY}  \|s - y\|_1 \, \gamma(s,y) \dif (s,y)
  \end{align}
  as desired, where line 3 is due to $\Pr(f^*(X)=s, h(X) = y, X=x) = \1[f^*(x)=s] \Pr(h(X)=y, X=x)$ for all $(s,y,x)$. 

For the converse, let a coupling $\gamma\in\Gamma(r^*,q)$ be given.  We show that the Markov kernel of the randomized classifier $h$ constructed in the statement satisfies the equality $\gamma(s,y) = \int_{{f^*}^{-1}(s)} \calK(x,y)\dif \mu^X(x)$, then \cref{eq:equiv.full} will follow directly from the same arguments used in the previous part.  Let $s\in\Delta_k$ and $y\in\calY$, and $x'\in {f^*}^{-1}(s)$ arbitrary, then 
\begin{align}
  \gamma(s,y) 
  &= \frac{\gamma(s,y)}{\gamma(s,\calY)} \, \gamma(s,\calY) \\
  &= \frac{\gamma(f^*(x'),y)}{\gamma(f^*(x'),\calY)} \, \gamma(s,\calY) \\
  &= \calK(x', y) \, \gamma(s,\calY) \\
  &= \calK(x', y) \, r^*(s) \\
  &= \calK(x', y) \int_{x\in {f^*}^{-1}(s)} \dif \mu^X(x) \\
  &=  \int_{x\in {f^*}^{-1}(s)} \calK(x', y) \dif \mu^X(x)\\
  &=  \int_{x\in {f^*}^{-1}(s)} \calK(x, y) \dif \mu^X(x),
\end{align}
where line 3 is by construction of $\calK$, line 4 from $\gamma\in\Gamma(r^*,q)$, and the last line is because $\calK(x, y)$ is constant for all $x\in{f^*}^{-1}(s)$, also by construction.
\end{proof}

\begin{proof}[Proof of \cref{thm:tradeoffs}]
\Cref{lem:equiv} implies that for each $a\in[m]$ and fixed $q_a\in\calQ_k$, the minimum error rate on group $a$, denoted by $\err_a$, among randomized classifiers $h_a:\calX\rightarrow\calY$ whose output distribution equals to $q_a$, is given by
\begin{align}
    \min_{h_a:h_a\sharp \mu^X_a= q_a}  \err_a(h_a)  
    \coloneqq \min_{h_a:h_a\sharp \mu^X_a=q_a}  \Pr(h_a(X)\neq Y \mid A=a)
    = \frac12  \,W_1(r^*_a,q_a).
\end{align}

Because of attribute-awareness, we can optimize each component of $h:\calX\times\calA\rightarrow\calY$, $h(\cdot,a)\eqqcolon h_a$ for all $a\in[m]$, independently. So for any set of fixed $q_1,\cdots,q_m\in\calQ_k$,
\begin{equation}
    \min_{h:h_a\sharp \mu^X_a= q_a,\forall a} \err(h)= \sum_{a\in[m]}   \min_{h_a:h_a\sharp \mu^X_a= q_a} w_a \err_a(h_a)  
    =    \sum_{a\in[m]} \frac{w_a}2\,W_1(r^*_a,q_a).
\end{equation}

Incorporating the $\alpha$-DP constraint, we get
\begin{align}
  \min_{h: \DPGap(h)\leq\alpha} \err(h) 
  &= \min_{\substack{q_1,\cdots,q_m\in\calQ_k\\\max_{a,a'}\|q_a-q_{a'}\|_\infty\leq\alpha}}\min_{h:h_a\sharp \mu^X_a= q_a,\forall a} \err(h)  = \min_{\substack{q_1,\cdots,q_m\in\calQ_k\\\max_{a,a'}\|q_a-q_{a'}\|_\infty\leq\alpha}} \sum_{a\in[m]}  \frac{w_a}2\,W_1(r^*_a,q_a). \tag*{\qedhere}
\end{align}
\end{proof}

\begin{proof}[Proof of \cref{thm:post.proc}]
By construction, the Markov kernel of the randomized optimal fair classifier $\bar h^*(x,a)\coloneqq\calT^*_{r^*_a\rightarrow q^*_a}\circ f^*_a(x)$ is 
\begin{equation}
\calK((x,a),y)=\frac{\gamma_a^*(f^*_a(x),y)}{\gamma_a^*(f^*_a(x), \calY)}
\end{equation}
where $\gamma^*_a\in\Gamma(r^*_a,q^*_a)$ is the optimal transport between $r^*_a$ and $q^*_a$.

We verify that the output distributions of $\bar h^*$ equal $q^*_1,\cdots,q^*_m$, thereby it is $\alpha$-DP because the $q^*_a$'s satisfy the constraint in \cref{eq:barycenter}, and its error rate achieves the minimum in \cref{thm:tradeoffs}.

First, for all $y\in\calY$,
  \begin{align}
    \Pr(\bar h^*(X,A)=y\mid A=a) 
    &=  \int_{\calX} \Pr(\bar h^*(x,a) = y) \Pr_{\mu}(X=x \mid A=a)  \dif x \\
    &= \int_\calX\calK((x,a),y)\dif \mu^X_a(x) \\
    &=\int_{\calX} \frac{\gamma_a^*(f^*_a(x),y)}{\gamma_a^*(f^*_a(x), \calY)} \dif \mu^X_a(x) \\
    &= \int_{\Delta_k}\frac{\gamma_a^*(s,y)}{\gamma_a^*(s, \calY)} \rbr*{\int_{{f_a^*}^{-1}(s)} \dif \mu^X_a(x)}\dif s\\
    &= \int_{\Delta_k}\frac{\gamma_a^*(s,y)}{\gamma_a^*(s, \calY)} r^*_a(s)\dif s\\
    &= \int_{\Delta_k}\gamma_a^*(s,y)\dif s = \gamma_a^*(\Delta_k,y) = q^*_a(y),
  \end{align}
where line 3 is due to $\gamma^*_a(f^*_a(x),y)=\gamma^*_a(s,y)$ being constant for all $x\in {f_a^*}^{-1}(s)$.

Similarly, \cref{lem:equiv.full} implies that the error rate on group $a$, denoted by $\err_a(\bar h^*)$, is 
\begin{equation}
  \err_a(\bar h^*) = \frac12\int_{\Delta_k\times\calY} \|s - y\|_1 \dif\gamma_a(s,y),
  \end{equation}
where $\gamma_a\in\Gamma(r^*_a,q^*_a)$ equals to
  \begin{align}
    \gamma_a(s,y) &= \int_{{f_a^*}^{-1}(s)} \calK((x,a),y)\dif \mu^X_a(x) \\
    &= \int_{{f_a^*}^{-1}(s)} \frac{\gamma_a^*(f^*_a(x),y)}{\gamma_a^*(f^*_a(x), \calY)}\dif \mu^X_a(x) \\
    &= \frac{\gamma_a^*(s,y)}{\gamma_a^*(s, \calY)} \int_{{f_a^*}^{-1}(s)} \dif \mu^X_a(x) \\
    &= \gamma_a^*(s,y).
  \end{align}
  So $\err_a(\bar h^*) = \frac12\int_{\Delta_k\times\calY} \|s - y\|_1 \dif\gamma^*_a(s,y) = \frac12 \,W_1(r^*_a,q^*_a)$ because $\gamma^*_a$ is an optimal transport  between $r^*_a$ and $q^*_a$, and $\err(\bar h^*)=\sum_{a\in[m]}w_a\err_a(\bar h^*)=\sum_{a\in[m]}\frac{w_a}2 \,W_1(r^*_a,q^*_a)$, the minimum error rate under $\alpha$-DP.
\end{proof}

\begin{proof}[Proof of \cref{prop:error.propa}]
Recall that the $\alpha$-fair classifier $\bar h$ returned from \cref{alg:post} is
\begin{equation}
  \bar h(x,a)=\calT^*_{r_a\rightarrow q_a}\circ f_a(x),
\end{equation}
where $(q_1,\cdots,q_m)$ is a minimizer of \cref{eq:barycenter} on $r_a\coloneqq f_a\sharp\mu^X_a$ and $\alpha$, and $\calT^*_{r_a\rightarrow q_a}$ is the optimal transport from $r_a$ to $q_a$.

Denote the $L^1$ difference between $f$ and $f^*$ by
\begin{equation}
  \calE\coloneqq \E_{\mu^X}[ \enVert{ f_A(X) - f^*_A(X)}_1 ].
\end{equation}

For the upper bound, by \cref{lem:rewrite} and the triangle inequality,
\begin{align}    
    \err(\bar h)   
    &= \frac1{2} \sum_{a\in[m]} w_a \E \sbr* { \enVert{ \calT^*_{r_a\rightarrow q_a}\circ f_a(X) - f^*_a(X)}_1  \mid A=a} \\
    &\leq \frac1{2}  \sum_{a\in[m]} w_a \rbr*{ \E \sbr{ \enVert{ \calT^*_{r_a\rightarrow q_a}\circ f_a(X) - f_a(X)}_1  \mid A=a} + \E \sbr{ \enVert{ f_a(X) - f^*_a(X)}_1  \mid A=a} } \\
    &= \sum_{a\in[m]} \frac{w_a}2\,W_1(r_a, q_a) + \frac \calE2 \\
    &= \min_{\substack{q'_1,\cdots,q'_m\in\calQ_k\\\max_{a,a'}\|q'_a-q'_{a'}\|_\infty\leq\alpha}}  \sum_{a\in[m]} \frac{w_a}2\,W_1(r_a, q'_a) + \frac\calE2,
\end{align}
where line 3 is because $\calT^*_{r_a\rightarrow q_a}$ is the optimal transport from $r_a$ to $q_a$ under the $\ell_1$ cost, and line 4 is because $(q_1,\cdots,q_m)$ is a minimizer. Let $(q^*_1,\cdots,q^*_m)$ denote a minimizer of \cref{eq:barycenter} on the distributions of Bayes scores $r^*_a\coloneqq f^*_a\sharp\mu^X_a$ and $\alpha$, then by \cref{thm:tradeoffs},
\begin{align}
    \err(\bar h)  -  \err^*_\alpha 
     &\leq \frac1{2}  \rbr*{\min_{\substack{q'_1,\cdots,q'_m\in\calQ_k\\\max_{a,a'}\|q'_a-q'_{a'}\|_\infty\leq\alpha}}\sum_{a\in[m]} w_a\,W_1(r_a, q'_a) - \sum_{a\in[m]} w_a\,W_1(r^*_a, q^*_a)} + \frac\calE2 \\
    &\leq \frac1{2} \rbr*{\sum_{a\in[m]} w_a\,W_1(r_a, q^*_a) - \sum_{a\in[m]} w_a\,W_1(r^*_a, q^*_a)} + \frac\calE2 \\
    &\leq  \sum_{a\in[m]} \frac{w_a}2\, W_1(r_a, r^*_a) + \frac\calE2 \\
    &\leq \frac12  \sum_{a\in[m]} w_a\E \sbr{ \enVert{ f_a(X) - f^*_a(X)}_1  \mid A=a}  + \frac\calE2 \\
    &= \calE,
\end{align}
where the last line is because for each $a\in[m]$, $W_1(r_a, r^*_a)$ is upper bounded by the transportation cost under the coupling given by the joint distribution of $(f_a(X), f^*_a(X))$ conditioned on $A=a$: denote the coupling by $\pi_a$, then clearly $\pi_a\in\Gamma(r_a,r^*_a)$, and $\int_{\Delta_k\times \Delta_k} \|s-s'\|_1\dif\pi_a(s,s') = \int_\calX\|f_a(x)-f^*_a(x)\|_1\dif\mu^X_a(x) = \E \sbr{ \enVert{ f_a(X) - f^*_a(X)}_1  \mid A=a}$.

For the lower bound, again by \cref{lem:rewrite},
\begin{align}    
  \err(\bar h)
    &= \sum_{a\in[m]}\frac{w_a}2\int_{\Delta_k\times\calY} \|s - y\|_1 \Pr(f_a^*(X)=s, \calT^*_{r_a\rightarrow q_a}\circ f_a(X) = y\mid A=a) \dif (s,y) \\
    &\geq \sum_{a\in[m]} \frac{w_a}2\,W_1(r^*_a,  q_a) \\
    &\geq \min_{\substack{q'_1,\cdots,q'_m\in\calQ_k\\\max_{a,a'}\|q'_a-q'_{a'}\|_\infty\leq\alpha}} \sum_{a\in[m]} \frac{w_a}2\,W_1(r^*_a, q'_a)  \\
    &=  \err^*_\alpha,
\end{align}
where line 2 is because the joint distribution of $(f_a^*(X), \calT^*_{r_a\rightarrow q_a}\circ f_a(X))$ conditioned on $A=a$ is a coupling belonging to $\Gamma(r^*_a, q_a)$, thereby the transportation cost represented by the quantity in the preceding line upper bounds $W_1(r^*_a,  q_a)$.
\end{proof}

\section{Proofs for Section~\ref{sec:algs}}

We establish the sample complexities in \cref{thm:finite.sample.discrete,thm:finite.sample.continuous} and the error propagation bound for smoothing (\cref{thm:error.smoothing}), in that order.  We remark that \cref{ass:calibration} of group-wise calibration of the scores can be dropped by adding the error propagation of \cref{prop:error.propa} to the results.  The sample complexity for the general case procedure described in \cref{sec:mixed} via smoothing can also be obtained by combining \cref{thm:finite.sample.continuous,thm:error.smoothing} (provided that the supports of the distributions after smoothing are contained in the simplex).

The generalization bound of the finite case uses an $\ell_1$ (TV) convergence result of empirical distributions, which follows directly from the concentration of multinoulli random variables~\citep{weissman2003InequalitiesL1Deviation}:

\begin{theorem} \label{thm:multinomial}
Let $p\in\Delta_d$, $d\geq2$, and $\hat p_n\sim\frac1n\mathrm{Multinomial}(n,p)$. W.p.\ at least $1-\delta$, $\|p-\hat p_n\|_1\leq \sqrt{{2d\ln(2/\delta)}/{n}}$.
\end{theorem}

\begin{corollary}\label{lem:conv.pmf}
  Let $p$ be a distribution over $\calX$ with finite support, and $x_1,\cdots,x_n\sim p$ be i.i.d.~samples.  Define the empirical distribution $\hat p_n \coloneqq  \frac1n\sum_{i=1}^n \delta_{x_i}$. W.p.\ at least $1-\delta$ over the random draw of the samples, $\|p-\hat p_n\|_1\leq \sqrt{{2|\calX|\ln(2/\delta)}/{n}}$.
\end{corollary}

\begin{theorem}[Hoeffding's Inequality]\label{thm:hoeffding}
  Let $x_1,\cdots,x_n\in\RR$ be i.i.d.\ random variables s.t.\ $a_i\leq x_i\leq b_i$ almost surely. W.p.\ at least $1-\delta$, $\envert{ \frac1n\sum_{i=1}^n (x_i - \E x_i)  } \leq \sqrt{\sum_{i=1}^n(b_i-a_i)^2/2n^2 \cdot \ln2/\delta}$.
\end{theorem}

\begin{proof}[Proof of \cref{thm:finite.sample.discrete}]
We first bound the error rate, then $\DPGap$.

\paragraph{Error Rate.} 
Consider the classification problem $\mu'$ derived from the original $\mu$ under an input transformation given by the joint distribution of $(f_A(X), A, Y)$, as discussed in \cref{sec:post.proc}, on which $\Id$ is the Bayes optimal score due to calibration of $f$'s.  Then by \cref{lem:rewrite} applied on $\mu'$, 
\begin{align}
  \err(\hat h) &=  \frac1{2}\sum_{a\in[m]}w_a\sum_{s\in\calR_a}\sum_{y\in\calY} \|s - y\|_1 \Pr( \Id(X')=s, \calT^*_{\hat r_a\rightarrow \hat q_a}(X') = y ) \\
  &=\frac1{2}\sum_{a\in[m]}w_a\sum_{s\in\calR_a}\sum_{y\in\calY} \|s - y\|_1 \, r_a(s) \, \Pr(\calT^*_{\hat r_a\rightarrow \hat q_a}(s)=y) \\
  &\leq  \frac1{2}\sum_{a\in[m]}w_a\sum_{s\in\calR_a}\sum_{y\in\calY} \|s - y\|_1 \, (\hat r_a(s) + |r_a(s)-\hat r_a(s)|)\, \Pr(\calT^*_{\hat r_a\rightarrow \hat q_a}(s)=y) \\
  &\leq  \sum_{a\in[m]} w_a \rbr*{\frac 12\, W_1(\hat r_a,\hat q_a) + \sum_{s\in\calR_a}\envert{r_a(s)-\hat r_a(s)}},
\end{align}
where line 4 uses the fact that each $\calT^*_{\hat r_a\rightarrow \hat q_a}$ is an optimal transport from $\hat r_a$ to $\hat q_a$.  Let $(q_1,\cdots,q_m)$ denote a minimizer of \cref{eq:barycenter} on the $r_a$'s with $\alpha$, then
\begin{align}
  \err(\hat h) - \err^*_{\alpha,f} 
      &\leq \sum_{a\in[m]}w_a \rbr*{\frac12\rbr*{W_1(\hat r_a,\hat q_a) - W_1(r_a,q_a)} +   \sum_{s\in\calR_a} |r_a(s)-\hat r_a(s)|} \\
      &= O\rbr*{ \sum_{a\in[m]}\hat w_a \rbr*{\frac12\rbr*{W_1(\hat r_a,\hat q_a) - W_1(r_a,q_a)} +   \sum_{s\in\calR_a} |r_a(s)-\hat r_a(s)|}} \\
      &\leq O\rbr*{ \sum_{a\in[m]}\hat w_a \rbr*{\frac12\rbr*{W_1(\hat r_a,q_a) - W_1(r_a,q_a)} +   \sum_{s\in\calR_a} |r_a(s)-\hat r_a(s)|} }\\
      &\leq O\rbr*{\sum_{a\in[m]}\hat w_a \rbr*{\frac12\,W_1(\hat r_a,r_a) +   \sum_{s\in\calR_a} |r_a(s)-\hat r_a(s)|}},
\end{align}
where we defined $\hat w_a=n_a/n$, line 2 is because $w_a=\Theta(\hat w_a)$ for all $a\in[m]$ when $n\geq \Omega(\max_a\ln(m/\delta)/w_a)$,  and line 3 is due to $(\hat q_1,\cdots,\hat q_m)$ being a minimizer of \cref{eq:barycenter} on the $\hat r_a$'s.  Because $\|s-s'\|_1\leq 2\, \1[s\neq s']$, by \cref{prop:tv.w1.01}, $W_1(\hat r_a,r_a)\leq \|\hat r_a-r_a\|_1$, so it follows that
\begin{align}
  \err(\hat h) - \err^*_{\alpha,f} 
  &\leq  O\rbr*{\sum_{a\in[m]} \sum_{s\in\calR_a} \hat w_a \envert{r_a(s)-\hat r_a(s)} } \\
  &\leq O\rbr*{\sum_{a\in[m]} \hat w_a \sqrt{\frac{|\calR_a|\ln\rbr{{m}/{\delta}}}{n_a}} } \\
  &\leq  O\rbr*{\max_a\sqrt{\frac{|\calR_a|\ln\rbr{{m}/{\delta}}}{nw_a}}  }
\end{align}
w.p.~at least $1-\delta$ from $m$ applications of \cref{lem:conv.pmf} and a union bound.

\paragraph{Fairness.} 
In the finite case, we can get a stronger result in terms of the $\ell_1$-norm. Note that for all $a\in[m]$ and $y\in\calY$,
\begin{align}
\MoveEqLeft \sum_{y\in\calY}\envert*{ \Pr(\hat h(X,A)=y\mid A=a) - \hat q_a(y)} \\
 &= \sum_{y\in\calY}\envert*{ \sum_{s\in\calR_a}r_a(s)\Pr(\calT^*_{\hat r_a\rightarrow \hat q_a}(s)=y)  - \sum_{s\in\calR_a}\hat r_a(s)\Pr(\calT^*_{\hat r_a\rightarrow \hat q_a}(s)=y)}\\
 &\leq \sum_{y\in\calY} \sum_{s\in\calR_a}\envert*{ r_a(s)  - \hat r_a(s)} \Pr(\calT^*_{\hat r_a\rightarrow \hat q_a}(s)=y)\\
 &= \sum_{s\in\calR_a}\envert*{ r_a(s)  - \hat r_a(s)} \\
  &\leq  O\rbr*{\sqrt{\frac{|\calR_a|\ln\rbr{{m}/{\delta}}}{nw_a}}  }
\end{align}
w.p.~at least $1-\delta$ from applications of \cref{lem:conv.pmf} and a union bound, where the first equality is because $\calT^*_{\hat r_a\rightarrow \hat q_a}$ is a transport from $\hat r_a$ to $\hat q_a$. This $\ell_1$-norm bound directly implies an $\ell_\infty$-norm bound:
\begin{align}
  \max_{y\in\calY}\envert*{ \Pr(\hat h(X,A)=y\mid A=a) - \hat q_a(y)}
  &\leq \sum_{y\in\calY}\envert*{ \Pr(\hat h(X,A)=y\mid A=a) - \hat q_a(y)} \\
  &=  O\rbr*{\sqrt{\frac{|\calR_a|\ln\rbr{{m}/{\delta}}}{nw_a}}  }.
\end{align}
Lastly, because of the constraint in \cref{eq:barycenter} that $\max_{a,a'\in[m]}\|\hat q_a-\hat q_{a'}\|_\infty\leq\alpha$,
\begin{align}
  \DPGap(\hat h)
  &= \max_{\substack{a,a'\in[m]\\ y\in\calY}}\envert*{ \Pr(\hat h(X,A)=y\mid A=a) - \Pr(\hat h(X,a')=y\mid A=a')} \\
  &\leq \begin{multlined}[t]  
  \max_{\substack{a,a'\in[m]\\ y\in\calY}} \envert*{\hat q_a(y) - \hat q_{a'}(y)} \\[-0.1em]
+\max_{\substack{a,a'\in[m]\\ y\in\calY}}\rbr*{\envert*{ \Pr(\hat h(X,A)=y\mid A=a) - \hat q_a(y)} + \envert*{\Pr(\hat h(X,a')=y\mid A=a')-\hat q_{a'}(y)}}
 \end{multlined}\\
 &\leq \alpha + \max_{a,a'\in[m]} O\rbr*{\sqrt{\frac{|\calR_a|\ln\rbr{{m}/{\delta}}}{nw_a}}  + \sqrt{\frac{|\calR_{a'}|\ln\rbr{{m}/{\delta}}}{nw_{a'}}} }.
\end{align}

The theorem then follows from a final application of union bound.
\end{proof}

The proof of the sample complexities in the continuous case uses the following uniform convergence results with the VC dimension and the pseudo-dimension as the complexity measure.  We omit the proofs, but refer readers to \citep[Theorem 6.8]{shalev-shwartz2014UnderstandingMachineLearning} and \citep[Theorem 11.8]{mohri2018FoundationsMachineLearning}, respectively.   We also need a characterization of the disagreement between the empirical transports $\calT^*_{\hat r_a\rightarrow\hat q}$ and the transport mappings $\calT_a\in\calG_k$ extracted from them on \cref{lst:ln.6,lst:ln.7,lst:ln.8,lst:ln.9} of \cref{alg:cont}, (to be) stated in \cref{cor:transport}.

\begin{theorem}[VC Dimension Uniform Convergence]\label{thm:vc}
  Let $\calH$ be a class of binary functions from $\calX$ to $\{0,1\}$, $p$ a distribution over $\calX\times\{0,1\}$, of which $(x_1,y_1),\cdots,(x_n,y_n)\sim p$ are i.i.d.\ samples.  W.p.\ at least $1-\delta$ over the random draw of the samples, $\forall h\in\calH$,
\begin{equation}
  \envert*{\E_{(X,Y)\sim p}\1[h(X)\neq Y] - \frac1n\sum_{i=1}^n \1[h(x_i)\neq y_i] }\leq  c \sqrt{\frac{d + \ln(1/\delta)}{n}}
\end{equation}
for some universal constant $c$, where $d$ is the VC dimension of $\calH$  (\cref{def:nat} with $k=2$).
\end{theorem}

\begin{theorem}[Pseudo-Dimension Uniform Convergence]\label{thm:pd}
Let $\calH$ be a class of functions from $\calX$ to $\RR$, $\ell:\calX\times\calY\rightarrow\RR_{\geq0}$ a nonnegative loss function upper bounded by $M$, $p$ a distribution over $\calX\times\calY$, of which $(x_1,y_1),\cdots,(x_n,y_n)\sim p$ are i.i.d.\ samples.  W.p.\ at least $1-\delta$ over the random draw of the samples, $\forall h\in\calH$,
\begin{equation}
  \envert*{ \E_{(X,Y)\sim p} \ell(h(X),Y) - \frac1n\sum_{i=1}^n \ell(h(x_i),y_i) } \leq cM\sqrt{\frac{d+\ln(1/\delta)}{n}}
\end{equation}
for some universal constant $c$, where is the pseudo-dimension of $\{(x,y)\mapsto \ell(h(x),y):h\in\calH\}$ (\cref{def:pdim}).
\end{theorem}

One highlight of our proof is that we avoided using the convergence of the empirical measure under Wasserstein distance in our arguments, which would have resulted in sample complexity that is exponential in the number of label classes $k$~\citep{weed2019SharpAsymptoticFinitesample}.  Instead, we leveraged the existence and uniqueness of the semi-discrete simplex-vertex optimal transport in the low complexity function class $\calG_k$, established in \cref{lem:monge.exist,thm:natdim}, whereby we can apply the above uniform bound to $\calG_k$ and achieve a rate that is only polynomial in $k$.

In addition, we remark that the $O(k^2/nw_a)$ term coming from the disagreements between $\calT_a$ and $\calT^*_{\hat r_a\rightarrow \hat q_a}$ on $(x_{a,i})_{i\in[n_a]}$ could potentially be improved to $O(k/nw_a)$ if LP is assumed to return an extremal solution~\citep{peyre2019ComputationalOptimalTransport}, and $\calG_k$ is modified so that the output on points that lie on each boundary can be specified, rather than always tie-broken to the $e_i$ with the largest index $i$.

\begin{proof}[Proof of \cref{thm:finite.sample.continuous}]

We first bound the error rate, followed by $\DPGap$. Recall that the classifier returned from \cref{alg:cont} is
\begin{equation}
  \hat h(x,a) \coloneqq \calT_a\circ f_a(x),
\end{equation}
where each $\calT_a\in\calG_k$ is extracted from the empirical optimal transport $\calT^*_{\hat r_a\rightarrow \hat q_a}$ by \cref{alg:cont}, obtained from calling $\LP(\hat r_1,\cdots,\hat r_m,\alpha)$, where $\hat q_1,\cdots,\hat q_m$ is the minimizer of \cref{eq:barycenter} on the $\hat r_a$'s with $\alpha$.  We will use a complexity result of $\calG_k$ in terms of its pseudo-dimension when associated with $\ell_1$ loss, and a VC bound of binarized versions of $\calG_k$ (to be defined in \cref{eq:bin.g}), which are deferred to \cref{cor:pdim,cor:vcdim}.

\paragraph{Error Rate.}
Consider the classification problem $\mu'$ derived from the original $\mu$ under an input transformation given by the joint distribution of $(f_A(X), A, Y)$, as discussed in \cref{sec:post.proc}, on which $\Id$ is the Bayes optimal score due to calibration of $f$'s.  Then by \cref{lem:rewrite} applied on $\mu'$, 
\begin{align}
  \err(\hat h)
  &= \sum_{a\in[m]}\frac{w_a}2\int_{\Delta_k\times\calY} \|s - y\|_1 \Pr( \Id(X')=s, \calT_a(X') = y ) \dif(s,y) \\
  &= \sum_{a\in[m]}\frac{w_a}2 \E_{S\sim r_a}[ \|\calT_{a}(S)-S\|_1].
\end{align}

Define $s_{a,j}\coloneqq f_a(x_{a,j})$. By \cref{thm:pd,cor:pdim}, and a union bound, we have w.p.~at least $1-\delta$, for all $a\in[m]$,
\begin{align}
  \E_{S\sim r_a}[ \|\calT_{a}(S)-S\|_1] - \frac1{n_a}\sum_{j=1}^{n_a} \|\calT_{a}(s_{a,j})-s_{a,j}\|_1  
  \leq O\rbr*{\sqrt{\frac{k+\ln (m/\delta)}{n_a}} }.
\end{align}
Because each $\calT_a$ is extracted from the empirical optimal transport $\calT^*_{\hat r_a\rightarrow \hat q_a}$ using \cref{lst:ln.6,lst:ln.7,lst:ln.8,lst:ln.9} of \cref{alg:cont}, by the discussion in \cref{sec:transport} and \cref{cor:transport}, they both agree on all of $(x_{a,i})_{i\in[n_a]}$ except for points that lie on the decision boundaries of $\calT_a$.  The boundaries are described by $\binom{k}{2}$ hyperplanes, and because $r_a$ is continuous, no two points in $(x_{a,i})_{i\in[n_a]}$ lie on the same hyperplane almost surely, so the number of disagreements is at most $\binom{k}{2}=k(k-1)/2$, and
\begin{align} 
  \envert*{ \frac1{n_a}\sum_{j=1}^{n_a} \|\calT_{a}(s_{a,j})-s_{a,j}\|_1 - W_1(\hat r_a, \hat q_a) }
  &=\envert*{ \frac1{n_a}\sum_{j=1}^{n_a} \|\calT_{a}(s_{a,j})-s_{a,j}\|_1 - \frac1{n_a}\sum_{j=1}^{n_a} \|\calT^*_{\hat r_a\rightarrow \hat q_a}(s_{a,j})-s_{a,j}\|_1 } \\
  &\leq \frac1{n_a}\sum_{j=1}^{n_a} \|\calT_{a}(s_{a,j})- \calT^*_{\hat r_a\rightarrow \hat q_a}(s_{a,j})\|_1  \\
  &= \frac2{n_a}\sum_{j=1}^{n_a} \1[\calT_{a}(s_{a,j})\neq\calT^*_{\hat r_a\rightarrow \hat q_a}(s_{a,j})] \\
  &\leq O\rbr*{\frac{k^2}{n_a}},
\end{align}
where the first equality is because $\calT^*_{\hat r_a\rightarrow \hat q_a}$ is the optimal transport from $\hat r_a$ to $\hat q_a$.  Therefore, we arrive at w.p.~at least $1-\delta$,
\begin{align}
  \err(\hat h) - \sum_{a\in[m]}\frac{w_a}2\, W_1(\hat r_a, \hat q_a) & \leq O\rbr*{ \sqrt{\frac{k+\ln (m/\delta)}{n_a}}  + \frac{k^2}{n_a} }\\
  &\leq O\rbr*{ \max_a \rbr*{\sqrt{\frac{k+\ln (m/\delta)}{n_a}}  + \frac{k^2}{n_a} }} \\
  &\eqqcolon \calE.
\end{align}

Continuing, let $\calT^*_{r_a\rightarrow q_a}\in\calG_k$ denote the optimal transport from $r_a$ to $q_a$, where the $q_a$'s denote the minimizer of \cref{eq:barycenter} on the $r_a$'s with $\alpha$.  The existence of this transport in $\calG_k$ is due to the problem being semi-discrete and \cref{lem:monge.exist}.  Define $q'_a(y) = \frac1{n_a} \sum_{j=1}^{n_a} \1[\calT^*_{r_a\rightarrow q_a}(s_{a,j})=y]$.  It follows that
\begin{align}
\err(\hat h) - \err^*_{\alpha,f} 
&\leq \calE + \sum_{a\in[m]} \frac{w_a}2\rbr*{ W_1(\hat r_a,\hat q_a) -  W_1(r_a,q_a)}\\
&= \calE + \sum_{a\in[m]} \frac{w_a}2\rbr*{\rbr{W_1(\hat r_a, \hat q_a) - W_1(\hat r_a,q'_a)} + \rbr{W_1(\hat r_a,q'_a) -  W_1(r_a,q_a)}}\\
&\leq  \calE + O\rbr*{ \sum_{a\in[m]}\frac{\hat w_a}2\rbr*{\rbr{W_1(\hat r_a, \hat q_a) - W_1(\hat r_a,q'_a)} + \rbr{W_1(\hat r_a,q'_a) -  W_1(r_a,q_a)}} } \\
&\leq  \calE + O\rbr*{ \sum_{a\in[m]}\frac{\hat w_a}2\rbr*{\rbr{W_1(\hat r_a,  q_a) - W_1(\hat r_a,q'_a)} + \rbr{W_1(\hat r_a,q'_a) -  W_1(r_a,q_a)}} } \\
&\leq \calE + O\rbr*{ \sum_{a\in[m]} \frac{\hat w_a}2 \rbr*{W_1( q_a,q'_a) + \rbr{W_1(\hat r_a,q'_a) -  W_1(r_a,q_a)}}},
\end{align}
where we defined $\hat w_a=n_a/n$, line 3 is because $w_a=\Theta(\hat w_a)$ for all $a\in[m]$ when $n\geq \Omega(\max_a\ln(m/\delta)/w_a)$, line 4 is due to $(\hat q_1,\cdots,\hat q_m)$ being a minimizer of \cref{eq:barycenter} on the $\hat r_a$'s

For the first term in the summand, because both distributions $q_a,q'_a$ are supported on the vertices, so by \cref{prop:tv.w1.01}, $W_1(q_a,q'_a) = \|q_a-q'_a\|_1$, and w.p.~at least $1-\delta$, for all $a\in[m]$,
\begin{align}
  \|q_a-q'_a\|_1
  &= \sum_{i\in[k]} \envert*{\E_{S\sim r_a}\1[\calT^*_{r_a\rightarrow q_a}(S)_i=1] -  \frac1{n_a} \sum_{j=1}^{n_a} \1[\calT^*_{r_a\rightarrow q_a}(s_{a,j})_i=1]}\\
  &= \enVert*{\E_{S\sim r_a}[\calT^*_{r_a\rightarrow q_a}(S)]-  \frac1{n_a} \sum_{j=1}^{n_a} \calT^*_{r_a\rightarrow q_a}(s_{a,j})}_1\\
  &\leq O\rbr*{\sqrt{\frac{k\ln\rbr*{{m}/{\delta}}}{n_a}}},
\end{align}
where the last line follows from \cref{thm:multinomial}, and a union bound over all $a\in[m]$.

For the second term, because then the joint probability of $(\widehat S,\calT^*_{r_a\rightarrow q_a}(\widehat S))$, $\widehat S\sim\hat r_a$, is a coupling belonging to $\Gamma(\hat r_a,q'_a)$, the transportation cost of $\calT^*_{r_a\rightarrow q_a}$ on $\hat r_a$ to $q'_a$ upper bounds $W_1(\hat r_a,q'_a)$, whereby w.p.~at least $1-\delta$,  for all $a\in[m]$,
\begin{align}
  W_1(\hat r_a,q'_a) -  W_1(r_a,q_a)
  &\leq \frac1{n_a}\sum_{j=1}^{n_a} \enVert*{\calT^*_{r_a\rightarrow q_a}(s_{a,j}) - s_{a,j} }_1 - \E_{S\sim r_a} \sbr*{\enVert*{\calT^*_{r_a\rightarrow q_a}(S) - S}_1}\\
  &\leq O\rbr*{\sqrt{\frac{\ln\rbr*{{m}/{\delta}}}{n_a}}}
\end{align}
by \cref{thm:hoeffding}, because $\|\calT^*_{r_a\rightarrow q_a}(s_{a,j}) - s_{a,j}\|_1\leq 2$.

Hence, putting everything together, we conclude with a union bound that
\begin{align}
    \err(\hat h) - \err^*_{\alpha,f} \leq O\rbr*{ \max_a \rbr*{\sqrt{\frac{k\ln (m/\delta)}{n_a}}  + \frac{k^2}{n_a} }}.
\end{align}

\paragraph{Fairness.} 
By applying \cref{thm:vc,cor:vcdim} to the artificial binary classification problem whose data distribution is the joint distribution of $(S,1)$, $S\sim r_a$, and a union bound, w.p.~at least $1-\delta$, for all $i\in[k]$ and $a\in[m]$,
\begin{align}
  \MoveEqLeft \envert*{ \Pr(\hat h(X,A)=e_i\mid A=a) - \frac1{n_a}\sum_{j=1}^{n_a} \1[ \calT_a(s_{a,j})=e_i]  } \\
  &= \envert*{ \E_{S\sim r_a,\calT_a}\1[ \calT_a(S)_i = 1]  - \frac1{n_a}\sum_{j=1}^{n_a} \1[ \calT_a(s_{a,j})_i=1]  } 
  \leq O\rbr*{\sqrt{\frac{k+\ln\rbr{{mk}/{\delta}}}{n_a}}}.
\end{align}

Now, the decision boundaries of the function $s\mapsto \1[\calT_{a}(s)_i = 1]\in\calG_{k,i}$ (defined in \cref{eq:bin.g})  are described by $k$ hyperplanes, and $\calT_a$ is extracted from $\calT^*_{\hat r_a\rightarrow \hat q_a}$, so by the discussion in \cref{sec:transport} and \cref{cor:transport} and the same reasoning used previously, they both agree on all but $k$ points in $(x_{a,i})_{i\in[n_a]}$ almost surely, thereby
\begin{align} 
  \envert*{\frac1{n_a}\sum_{j=1}^{n_a} \1[ \calT_{a}(s_{a,j})=e_i] - \hat q_a(e_i)}
  &=\envert*{ \frac1{n_a}\sum_{j=1}^{n_a} \1[ \calT_{a}(s_{a,j})_i=1] - \frac1{n_a}\sum_{j=1}^{n_a} \1[ \calT^*_{\hat r_a\rightarrow \hat q_a}(s_{a,j})_i=1] } \\
  &\leq \frac1{n_a}\sum_{j=1}^{n_a} \1[\calT_{a}(s_{a,j})_i\neq\calT^*_{\hat r_a\rightarrow \hat q_a}(s_{a,j})_i] 
  \leq O\rbr*{\frac{k}{n_a}}.
\end{align}

Therefore, we conclude that
\begin{align}
  \DPGap(\hat h)
  &= \max_{\substack{a,a'\in[m]\\ y\in\calY}}\envert*{ \Pr(\hat h(X,A)=y\mid A=a) - \Pr(\hat h(X,a')=y\mid A=a')} \\
  &\leq \begin{multlined}[t]  
  \max_{\substack{a,a'\in[m]\\ y\in\calY}} \envert*{\hat q_a(y) - \hat q_{a'}(y)} \\
+\max_{\substack{a,a'\in[m]\\ y\in\calY}} \rbr*{\envert*{ \Pr(\hat h(X,A)=y\mid A=a) - \hat q_a(y)} + \envert*{\Pr(\hat h(X,a')=y\mid A=a')-\hat q_{a'}(y)}} 
 \end{multlined}\\
  &\leq \alpha + O\rbr*{\max_a\rbr*{ \sqrt{\frac{k+\ln\rbr{{mk}/{\delta}}}{n_a}} + \frac k{n_a}}}.
\end{align}

The theorem them follows from noting that $n_a=\Theta(nw_a)$ when $n\geq \Omega(\max_a\ln(m/\delta)/w_a)$.
\end{proof}

\begin{proof}[Proof of \cref{thm:error.smoothing}]
Let the $(q_1,\cdots,q_m)$ denote a minimizer of \cref{eq:barycenter} on the $r_a$'s with $\alpha$, then
\begin{equation}
  \bar h_\rho(x,a)\coloneqq \calT^*_{\tilde r_a\rightarrow q_a}\circ u_\rho\circ f_a(x)
\end{equation}
where $\tilde r_a \coloneqq  u_\rho\sharp r_a$.  

Denote the coupling associated with $\calT^*_{\tilde r_a\rightarrow q_a}$ by $\gamma_a\in\Gamma(\tilde r_a,q_a)$, then the Markov kernel of $\calT^*_{\tilde r_a\rightarrow q_a}\circ u_\rho$ is
\begin{equation}
  \calK(s,T)= \E_{N\sim\rho}\sbr*{\frac{\gamma_{a}(s + N, T)}{\gamma_{a}(s + N, \calY)}} = \int_{\tilde s\in\RR^k} \frac{\gamma_{a}(\tilde s, T)}{\gamma_{a}(\tilde s, \calY)} \dif (\rho*\delta_s)(\tilde s) = \int_{\tilde s\in\RR^k} \frac{\gamma_{a}(\tilde s, T)}{\tilde r_{a}(\tilde s)} \dif (\rho*\delta_s)(\tilde s),
\end{equation}
where $*$ denotes convolution.

Consider the classification problem $\mu'$ derived from the original $\mu$ under an input transformation given by the joint distribution of $(f_A(X), A, Y)$, as discussed in \cref{sec:post.proc}, on which $\Id$ is the Bayes optimal score due to calibration of $f$.  Then by \cref{lem:equiv.full} applied on $\mu'$, the error rate on group $a$, denoted by $\err_a(\bar h_\rho)$, is 
\begin{equation}
  \err_a(\bar h_\rho) = \frac12\int_{\Delta_k\times\calY} \|s - y\|_1 \dif\gamma'_a(s,y), \label{eq:smooth.t}
  \end{equation}
where $\gamma'_a\in\Gamma(r_a,q_a)$ equals to
  \begin{align}
    \gamma'_a(s,y) &= \int_{\Id^{-1}(s)} \calK(s',y)\dif r_a(s') 
    =  \calK(s,y) \, r_a(s) 
    = \int_{\RR^k} \frac{\gamma_{a}(\tilde s, y)}{\tilde r_{a}(\tilde s)} \dif (\rho*\delta_s)(\tilde s) \, r_a(s),
  \end{align}
whereby
\begin{align}
   2\err_a(\bar h_\rho) 
   &= \sum_{y\in\calY} \int_{\Delta_k} \int_{\RR^k}  \|s - y\|_1 \frac{\gamma_{a}(\tilde s, y)}{\tilde r_{a}(\tilde s)} \,  (\rho*\delta_s)(\tilde s)  r_a(s) \dif \tilde s \dif s\\
    &\leq 
   \sum_{y\in\calY} \int_{\Delta_k} \int_{\RR^k}  \|\tilde s - y\|_1   \frac{\gamma_{a}(\tilde s, y)}{\tilde r_{a}(\tilde s)} \,  (\rho*\delta_s)(\tilde s)  r_a(s) \dif \tilde s \dif s \\ 
   &\qquad\qquad\qquad + \sum_{y\in\calY} \int_{\Delta_k} \int_{\RR^k}   \|s-\tilde s\|_1    \frac{\gamma_{a}(\tilde s, y)}{\tilde r_{a}(\tilde s)} \,  (\rho*\delta_s)(\tilde s)  r_a(s) \dif \tilde s \dif s \\
    &\eqqcolon  \sum_{y\in\calY}  \int_{\RR^k}  \|\tilde s - y\|_1   \frac{\gamma_{a}(\tilde s, y)}{\tilde r_{a}(\tilde s)} \,  \rbr*{\int_{\Delta_k} (\rho*\delta_s)(\tilde s)  r_a(s) \dif s} \dif \tilde s  \\ 
  &\qquad\qquad\qquad + \sum_{y\in\calY} \int_{\Delta_k} \int_{\RR^k}   \|n\|_1    \frac{\gamma_{a}(n-s, y)}{\tilde r_{a}(n-s)} \,  (\rho*\delta_s)(n-s)  r_a(s) \dif n \dif s \\
    &= 
   \sum_{y\in\calY}  \int_{\RR^k}  \|\tilde s - y\|_1   \frac{\gamma_{a}(\tilde s, y)}{\tilde r_{a}(\tilde s)} \,   \tilde r_a(\tilde s) \dif \tilde s    +  \int_{\Delta_k} \int_{\RR^k}   \|n\|_1  \dif  \rho(n)  \dif r_a(s)
\\
    &= 
   W_1(\tilde r_a,q_a)   + \E_{N\sim\rho} [\|N\|_1],
\end{align}
where line 3 involves a change of variable $n\coloneqq \tilde s-s$. Then we have
\begin{align}
  \err(\bar h_\rho) - \err^*_{\alpha,f} &\leq  \sum_{a\in[m]} \frac{w_a}2\rbr*{W_1(\tilde{r}_a, q_a) - W_1(r_a, q_a)}  +  \sum_{a\in[m]}  \frac {w_a}2 \E_{N\sim\rho} [\|N\|_1] \\
  &\leq  \sum_{a\in[m]} \frac{w_a}2\,W_1(\tilde{r}_a, r_a)  +   \frac12 \E_{N\sim\rho} [\|N\|_1].
\end{align}
Now, we upper bound the first term.  Consider the coupling $\pi_a\in\Gamma(\tilde{r}_a,r_a)$ given by $\pi_a(\tilde s, s)= \rho(\tilde s- s)r_a(s)$, whereby
\begin{align}
  W_1(\tilde{r}_a, r_a) 
  &= \inf_{\gamma\in\Gamma(\tilde{r}_a,r_a)} \int_{\RR^k\times\Delta_k} \|\tilde s-s\|_1\dif\gamma(\tilde s,s) \\
  &\leq \int_{\RR^k\times\Delta_k} \|\tilde s-s\|_1\dif\pi_a(\tilde s,s) \\
  &=  \iint \|\tilde s-s\|_1\rho(\tilde s-s)r_a(s) \dif \tilde s\dif s \\
  &\eqqcolon  \iint \|(s + n)-s\|_1\rho(n)r_a(s) \dif n\dif s \\
  &= \E_{N\sim\rho} [\|N\|_1].
\end{align}
Substituting this into the result above, we obtain the upper bound.

On the other hand, the lower bound follows from \cref{eq:smooth.t}, where
\begin{equation}
  \err(\bar h_\rho) = \sum_{a\in[m]}\frac{w_a}2\int_{\Delta_k\times\calY} \|s - y\|_1 \dif\gamma'_a(s,y) \geq \sum_{a\in[m]} \frac{w_a}2\,W_1(r_a,q_a) = \err^*_{\alpha,f}. \tag*{\qedhere}
\end{equation}
\end{proof}

\section{Optimal Transport Between Simplex and Vertex Distributions} \label{sec:transport}

The $(k-1)$-dimensional probability simplex is defined for $k\geq 2$ by
\begin{equation}
\Delta_k\coloneqq\cbr*{x\in\RR^k:\sum_{i=1}^k x_i=1, x_j\geq 0,\forall j\in[k]},
\end{equation}
and its $k$ vertices are $\{e_1,\cdots, e_k\}$.  In this section, we study the optimal transport problem between distributions supported on the simplex and its vertices under the $\ell_1$ cost, given by $c(x,y)=\|x-y\|_1$.

By extending each $\Delta_k$ to infinity, we obtain a $(k-1)$-dimensional affine space of
\begin{equation}
  \DD^k \coloneqq  \cbr*{x\in\RR^k:\sum_{i=1}^k x_i=1} \supset\Delta_k. \label{eq:affine}
\end{equation}
Define vectors
\begin{equation}  
v_{ij} \coloneqq  e_j - e_i, \quad \forall i,j\in[k],
\end{equation}
and note that for each $i\in[k]$, $\{v_{ij}:j\neq i\}$ forms a basis for $\DD^k$.  Also, observe the following identity for the $\ell_1$ distance between a point on the simplex and a point on the vertex:
\begin{equation}
  \|x-e_i\|_1 = 1-x_i + \sum_{j\neq i} x_j = 1 - 2x_i + \sum_{j} x_j = 2(1-x_i), \quad\forall x\in\Delta_k,\, i\in[k] \label{eq:simplex.dist}
\end{equation}
(this identity is central to some of the upcoming results).

A main result of this section is that when the transportation problem is semi-discrete, the deterministic (Monge) optimal transport exists, and is unique:

\begin{theorem}\label{lem:monge.exist}
Let $p$ be a continuous probability measure on $\Delta_k$, $q$ a probability measure on $\{e_1,\cdots,e_k\}$, and $c(x,y)=\|x-y\|_1$.  The optimal transport from $p$ to $q$ is a Monge plan, and is unique up to sets of measure zero w.r.t.~$p$.
\end{theorem}

Specifically, the optimal transport $\calT^*_{p\rightarrow q}$ in \cref{lem:monge.exist} is given by the $c$-transform of the Kantorovich potential from the Kantorovich-Rubinstein dual formulation of the transportation problem.  In other words, it belongs to the following parameterized class of deterministic functions:
\begin{equation}
  \calG_k \coloneqq  \cbr*{x\mapsto e_{\argmin_{i\in[k]}(\|x - e_i\|_1 -\psi_i)} : \psi\in\RR^k}\subset \{e_1,\cdots,e_k\}^{\Delta_k} \label{eq:g_k}
\end{equation}
(break ties to the tied $e_i$ with the largest index $i$).  

This function class is therefore of particular interest to various analyses in this paper.  For the generalization bounds in \cref{sec:continuous}, we show that this function class has low complexity in terms of the Natarajan dimension (\cref{def:nat}):

\begin{theorem} \label{thm:natdim} 
$\dN(\calG_k)=k-1$.
\end{theorem}


In addition, note that as illustrated in \cref{fig:nn}, we can equivalently characterize each $g\in\calG_k$ by the center point at which its $k$ decision boundaries all intersect:

\begin{proposition}\label{prop:alt}
  Define the function class $\calG_k'\subset\{e_1,\cdots,e_k\}^{\Delta_k}$ parameterized by $\RR^k$ s.t.~for each $g_z\in\calG'_k$ with parameter $z\in\RR^k$,
  \begin{equation}
    g_z(x) = e_i \quad\text{if}\quad x_j-x_i \leq z_j-z_i \iff x^\T v_{ij} \leq z^\T v_{ij},\quad\forall j\neq i \label{eq:g_k.alt}
  \end{equation}
  (when multiple $e_i$'s are eligible, output the tied $e_i$ with the largest index $i$).
  
  We have $\forall g_\psi\in\calG_k$, $g_\psi=g_z\in\calG'_k$ by setting
  \begin{equation}
    z_i=\frac1k+\frac12 \rbr*{\frac1k\sum_{j=1}^{k}\psi_j - \psi_i},\quad \forall i\in[k] \label{eq:alt.z}
  \end{equation}
  (the choice of $\sum_{i=1}^kz_i=1$ s.t.~$z\in\DD^k$ was arbitrary, due to an extra degree of freedom because the support of $g$ is contained in $\Delta_k$).
  
  Conversely, $\forall g_z\in\calG'_k$, $g_z=g_\psi\in\calG_k$ by setting
  \begin{equation}
    \psi_i=2(z_1-z_i) = 2z^\T v_{i1},\quad \forall i\in[k] \label{eq:alt.psi}
  \end{equation}
  (again, the choice of $\psi_1=0$ was arbitrary).
\end{proposition}

\begin{proof}
  Let $g_\psi\in\calG_k$, then for the $g_z\in\calG'_k$ constructed in \cref{eq:alt.z}, by \cref{eq:simplex.dist},

\begin{alignat}{3}
\text{$g_z(x)=e_i$ is eligible} \quad & \iff \quad &   x_j-x_i          &  \leq  z_j-z_i,                                                 && \quad\forall j\neq i \\
                                \quad & \iff \quad &   x_j-x_i          &  \leq  ( \psi_i - \psi_j)/2,  && \quad\forall j\neq i \\
                                \quad & \iff \quad & 2(x_j-x_i)         &  \leq  \psi_i -\psi_j,                                          && \quad\forall j\neq i \\
                                \quad & \iff \quad & 2(1-x_i) -\psi_i   &  \leq  2(1-x_j)-\psi_j,                                         && \quad\forall j\neq i \\
                                \quad & \iff \quad & \|x-e_i\|_1-\psi_i &  \leq  \|x-e_j\|_1-\psi_j,                                      && \quad\forall j\neq i.
\end{alignat}

Conversely, let $g_z\in\calG'_k$, then for the $g_\psi\in\calG_k$ constructed in \cref{eq:alt.psi},
\begin{alignat}{3}
\text{$g_\psi(x)=e_i$ is eligible} \quad & \iff \quad & \|x-e_i\|_1-\psi_i &  \leq  \|x-e_j\|_1-\psi_j,     && \quad\forall j\neq i  \\
                                   \quad & \iff \quad & 2(x_j-x_i)         &  \leq  2(z_1-z_i)- 2(z_1-z_j), && \quad\forall j\neq i  \\
                                   \quad & \iff \quad &   x_j-x_i          &  \leq  z_j-z_i,                && \quad\forall j\neq i.\tag*{\qedhere}
\end{alignat}
\end{proof}

We will often use this alternative characterization of $\calG_k$.

The remaining proofs are deferred to \cref{sec:transport.proof}.  \Cref{lem:monge.exist} is established via an analysis of the geometry of the simplex-vertex optimal transport, which we discuss in the next section.

\subsection{Geometry of Optimal Transport} \label{sec:ot.geometry}

Let $p$ be an arbitrary distribution supported on $\Delta_k$, and $q$ a (finite) distribution over $\calY\coloneqq\{e_1,\cdots,e_k\}$.  We study the geometric properties of the optimal solution to the (Kantorovich) transportation problem between $p,q$ under the $\ell_1$ cost,
\begin{equation}
  \sup_{\gamma\in\Gamma(p,q)}\int_{\Delta_k\times\calY} \|x-y\|_1 \dif \gamma(x,y), \label{eq:app.transport}
\end{equation}
and note that the supremum can be attained because the supports are compact.

\begin{figure}[t]
    \centering
    \includegraphics[width=\linewidth]{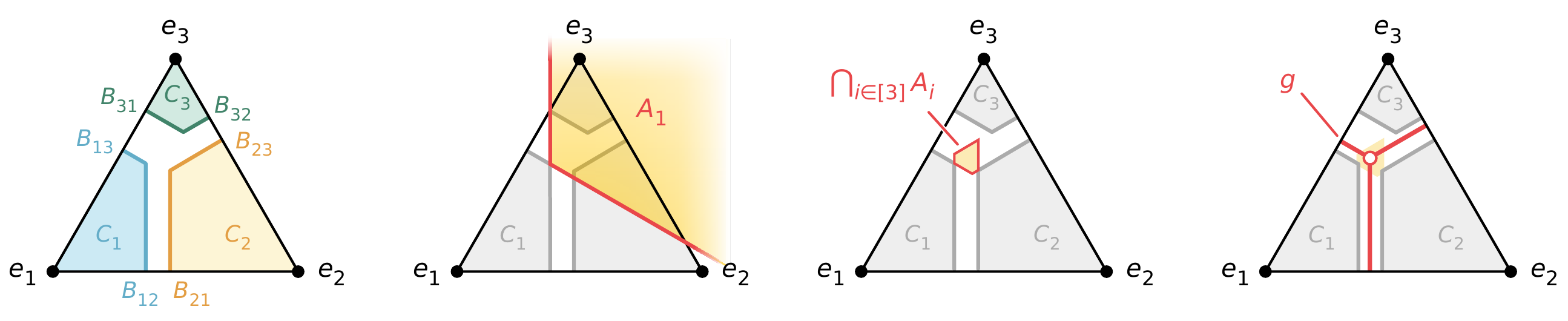}
    \caption{Illustration of the objects defined on \cref{eq:geometry.A,eq:decision.region} for $k=3$.  See \cref{fig:geo.ot} for an example where the intersection is empty, when the underlying transport is not optimal.}
    \label{fig:geo.def}
\end{figure}

First, given a simplex-vertex transport $\gamma\in\Gamma(p,q)$, we define the following geometric objects:
\begin{equation}\label{eq:decision.region}
\begin{gathered}
  B_{ij}\coloneqq \min\cbr*{b \in\RR : \gamma(\{  x\in\Delta_k: x^\T v_{ij} \leq b -1 \}, e_i) = q(e_i) }\cup\{0\},\quad \text{ and}\\
    C_{i}\coloneqq \bigcap_{j\neq i} \cbr{x\in\Delta_k: x^\T v_{ij} \leq B_{ij} -1}. 
\end{gathered}
\end{equation}
For each $i\in[k]$, $B_{ij}$ defines the (smallest offset of the) halfspace in the $v_{ij}$ direction in which all points that are transported by $\gamma$ to $e_i$ are contained, and $C_i$ is formed by the intersections of these halfspaces, also containing all points transported to $e_i$.  See \cref{fig:transport,fig:geo.def} for illustrations.

Now, if $\gamma^*$ is an optimal transport of \cref{eq:app.transport}, then intuition tells us that in order to achieve minimum cost, the halfspaces along each direction should not overlap  (i.e., $B_{ij}+B_{ji} \leq 2$ for all $i\neq j$), and the $C_i$'s should not intersect (except on a set of Lebesgue measure zero).  We show that these intuitions regarding the geometry of $\gamma^*$ are indeed valid, and they are implied by showing that the intersection of the following sets $A_i$ is nonempty (see \cref{fig:geo.def} for an illustration), 
\begin{align}
  A_{i}\coloneqq \bigcap_{j\neq i} \cbr{x\in\DD^k: x^\T v_{ij} \geq B_{ij} -1}. \label{eq:geometry.A}
\end{align}

\begin{proposition}\label{prop:ot.intersect}
  If $\gamma^*$ is a minimizer of \cref{eq:app.transport}, then $\bigcap_{i\in[k]} A_i\neq \emptyset$.
\end{proposition}
The proof is deferred to \cref{sec:transport.proof}.  Note that $\bigcap_{i\in[k]} A_i$ is exactly the set considered on \cref{lst:ln.7} of \cref{alg:cont}, and \cref{prop:ot.intersect} says that if $\gamma^*$ is an optimal transport, then a point $z\in\bigcap_{i\in[k]} A_i$ exists.  The significance of this point is that, the function $g_z\in\calG_k$ with parameter $z\in\DD^k$ agrees with the transport $\calT_{p\rightarrow q}$ associated with $\gamma^*$ only except for points that lie on the boundaries (which have Lebesgue measure zero):

\begin{lemma}\label{cor:transport}
Let $p,q$ be probability measures on $\Delta_k$ and $\{e_1,\cdots,e_k\}$, respectively.  If $\gamma^*\in\Gamma(p,q)$ is a minimizer of \cref{eq:app.transport}, then $\exists \calT\in\calG_k$ with parameters $z\in\DD^k$ satisfying
\begin{equation}
  \gamma(x, \calT(x)) = p(x), \quad \forall x\in\supp(p)\setminus \bigcup_{i\neq j}\{x\in\DD^k: x^\T v_{ij}=z^\T v_{ij}\}.
\end{equation}
\end{lemma}

This result underlies many discussions throughout our presentation: \textit{(i)}~the construction used in its proof led to \cref{lst:ln.6,lst:ln.7,lst:ln.8,lst:ln.9} of \cref{alg:cont} for extracting post-processing functions from the empirical optimal transports, \textit{(ii)}~it embodies the argument used in the proof of \cref{thm:finite.sample.continuous} regarding the disagreements between the extracted functions and the empirical transports, and \textit{(iii)}~the existence part of \cref{lem:monge.exist} is a direct consequence, since the set on which disagreements may occur always has measure zero when $p$ is continuous.

\begin{proof}
Let $z\in \bigcap_{i\in[k]} A_i$, which exists due to  \cref{prop:ot.intersect}.  Then let $\calT\in\calG_k$ with parameter $z$, which we show agrees with $\gamma^*$ on all  $x\in\supp(p)\setminus \bigcup_{i\neq j}\{x\in\DD^k: x^\T v_{ij}=z^\T v_{ij}\}$: suppose $\calT(x)=e_i$, then $\calT(x)=e_i\iff x^\T v_{ij}\leq z^\T v_{ij}$ by construction.  Furthermore, by the definition of $A_i$ in \cref{eq:geometry.A} of $\gamma^*$, $x^\T v_{ij} < z^\T v_{ij} \leq B_{ij} - 1$ for all $j\neq i$, so we must have that $\gamma^*(x,e_j)=0,\;\forall j\neq i\implies \gamma^*(x,e_i)=p(x)$.  Otherwise, it would contradict the definition of $B_{ij}$ in \cref{eq:decision.region}.
\end{proof}

\subsection{Omitted Proofs for Section~\ref{sec:transport}} \label{sec:transport.proof}

\begin{proof}[Proof of \cref{lem:monge.exist}]
For existence, \cref{cor:transport} provides a $\calT\in\calG_k$ that agrees with the optimal transport almost everywhere, since the set of points lying on the boundaries has measure zero w.r.t.~$p$ by continuity.

Next, we prove uniqueness.  Let $\gamma,\gamma'\in\Gamma(p,q)$ be two optimal transports,
and $\calT, \calT'\in\calG_k$ mappings provided by \cref{cor:transport} that agree with $\gamma,\gamma'$ a.e. We will show that $\calT=\calT'$ a.e., and so is $\gamma=\gamma'$.

Denote the parameter (i.e.,~center of the decision boundaries) of $\calT$ (analogously for $\calT'$) by $z\in\DD^k$, the decision boundaries by $B_{ij}\coloneqq z^\T v_{ij}+1$, and the decision regions by $C_{i}\coloneqq \bigcap_{j\neq i} \cbr{x\in\Delta_k: x^\T v_{ij} \leq B_{ij} -1}$. By definition of $\calG_k$, $\Delta_k=\bigsqcup_{i=1}^k C_i$, and for all $x\in\Delta_k$, $\calT(x)=e_i\iff x\in C_i$ almost surely, therefore, $\gamma(C_i,e_i)=q(e_i)$ because $\calT$ agrees with the transport $\gamma$ a.e. 

Define the difference in the boundaries between $\calT$ and $\calT'$ by $d_{ij}\coloneqq B'_{ij} - B_{ij}$, and note that
\begin{equation}
    d_{ij} = d_{nj} - d_{ni} \quad\text{with}\quad d_{\ell\ell}\coloneqq  0, \quad \forall i,j,n,\ell\in[k], \label{eq:diff}
\end{equation}
which follows from the observation that 
\begin{equation}
    B_{ij}=z^\T (v_{nj} - v_{ni}) + 1 = B_{nj} - B_{ni} + 1 \quad\text{with}\quad  B_{\ell\ell}\coloneqq 0,\quad \forall i,j,n,\ell\in[k].
\end{equation}
Construct a directed graph of $k$ nodes where $(i,j)$ is an edge iff $d_{ij}>0$.  Note that this graph is acyclic: first, it cannot contain cycles of length 2, otherwise, $(i,j),(j,i)\in E\implies d_{ij}+d_{ji}>0$ contradicts the fact that $d_{ij}+d_{ji}=0$ by definition; next, consider the shortest cycle, and let $(i,j),(n,i)$, $j\neq n$ denote two edges contained in it. It follows that $d_{ij},d_{ni}>0$, and $d_{nj}\leq0$, or it is not the shortest cycle. Then by \cref{eq:diff}, $0<d_{ij}=d_{nj}-d_{ni}< 0$, which is a contradiction.  

Now, we show by strong induction on the reverse topological order of the graph nodes that for all $i\in[n]$, $p(C_i\oplus C'_i)=0$ where $\oplus$ denotes the symmetric difference of the sets.
For the base case, let $i$ denote a sink node in the graph, then we have that $d_{ij}\leq 0$ for all $j$, meaning that $C'_i\subseteq C_i$.  Then $q(e_i)=\gamma'(C'_i,e_i)=p(C'_i)\leq p(C_i)=\gamma(C_i,e_i)=q(e_i)$.  
If the inequality is strict, then it is a contradiction; otherwise, combining the equality with $\calT(x)=e_i\iff x\in C_i$ a.s.~(and $\calT'$ analogously) implies $p(C_i\oplus C'_i)=p(C_i\setminus C'_i)=0$.
For the inductive case, let $i$ denote a node, and $J\subseteq [n]\setminus\{i\}$ the set of nodes directed to from $i$, then by construction $\bigsqcup_{j\in J\cup\{i\}}C'_j\subseteq \bigsqcup_{j\in J\cup\{i\}}C_j$. Let $F_i\coloneqq C_i\cap C'_i$, and note that for all $x\in C'_i\setminus F_i$, $\calT'(x)=e_i$ and $\calT(x)\in\{e_j:j\in J\setminus\{i\}\}$. Therefore, $C'_i\setminus F_i \in\bigcup_{j\in J} (C_j\oplus C'_j)$, and by the inductive hypothesis, $p(C'_i\setminus F_i)\leq0$. It then follows that $p(C'_i)\leq p(C_i)$, and subsequently $p(C_i\oplus C'_i)=p(C_i\setminus C'_i)=0$ by the same arguments used in the base case.

Therefore, $p(\{x:\calT(x)\neq \calT'(x)\})\leq \sum_{i=1}^k p(C_i\oplus C'_i)=0$, so $\calT=\calT'$ a.e.
\end{proof}

The proof of \cref{prop:ot.intersect} needs the following technical result, which at a high-level states that if a collection of $v_{ij}$-aligned convex sets do not intersect, then they cannot cover the entire space:

\begin{proposition}\label{prop:technical.set}
Let $B\in\RR^{k\times k}$ arbitrary, and define $S_i\coloneqq  \bigcap_{j\neq i} \cbr{x\in \DD^k: x^\T v_{ij} \leq B_{ij} -1}$ for each $i\in[k]$. We have $\bigcap_{i\in[k]} S_i= \emptyset\implies \bigcup_{i\in[k]} S_i\neq \DD^k$.
\end{proposition}

While this could be proved with elementary arguments, for clarity, we use known results from algebraic topology in the final steps of our proof.  The tools and concepts that we use include homotopy equivalence and homology groups (we omit the definition for the latter, but refer readers to \citep{spanier1981AlgebraicTopology} for a textbook).  The definitions are provided below for completeness; readers may skip to the main proof.

\begin{definition}[Homotopy]
  Let $\calX,\calY$ be topological spaces, and $f,g:\calX\rightarrow\calY$ two continuous functions.  A homotopy between $f$ and $g$ is a continuous function $h:\calX\times[0,1]\rightarrow\calY$, such that $h(x,0)=f(x)$ and $h(x,1)=g(x)$ for all $x\in\calX$.  We say $f,g$ are homotopic if there exists a homotopy between them.
\end{definition}

\begin{definition}[Homotopy Equivalence]
  Let $\calX,\calY$ be topological spaces. If there exist continuous maps $f:\calX\rightarrow\calY$ and $g:\calY\rightarrow\calX$ such that $g\circ f$ is homotopic to the identity map $\Id_\calX$ on $\calX$, and $f\circ g$ is homotopic to $\Id_\calY$, then $\calX$ and $\calY$ are homotopy equivalent, denoted by $\calX\cong\calY$.
\end{definition}

\begin{fact}[Homology]
(See~\citep{spanier1981AlgebraicTopology} for a textbook).
\begin{enumerate}
\item The homology groups of $\RR^d$, denoted by $H_n(\RR^d)$ for $n\in\{0,1,2,\cdots\}$, are
\begin{equation}
    H_n(\RR^d) = \begin{cases}
        \ZZ & \text{if $n=0$} \\
        \{0\} &\text{else.}
    \end{cases}
\end{equation}
\item The homology groups of the $d$-dimensional simplex, $\Delta_{d+1}$, are
\begin{equation}
    H_n(\Delta_{d+1}) = \begin{cases}
        \ZZ & \text{if $n=0$} \\
        \{0\} &\text{else.}
    \end{cases}
\end{equation}

\item The homology groups of the $d$-dimensional simplex without its interior, $\partial \Delta_{d+1}$, are
\begin{equation}
    H_n(\partial \Delta_{d+1}) = \begin{cases}
        \ZZ & \text{if $n=0$ or $d-1$} \\
        \{0\} &\text{else.}
    \end{cases}
\end{equation}

\item If $\calX\cong\calY$, then $H_n(\calX)=H_n(\calY)$ for all $n$.
\end{enumerate}
\end{fact}

Clearly, the affine space $\DD^{k}\cong\RR^{k-1}$ via a rotation and a translation.
We also cite the Nerve theorem~\citep[Theorem 3.1]{bauer2022UnifiedViewFunctorial}:

\begin{theorem}[Nerve]\label{thm:nerve}
  Let $\calS=\{S_1,\cdots,S_n\}$ be a finite collection of sets, and define its \textnormal{nerve} by
  \begin{equation}
    \Nrv(\calS) =  \cbr*{ J\subseteq [n] : \bigcap_{i\in J} S_i \neq \emptyset }.
  \end{equation}
  If the sets $S_i$'s are convex closed subsets of $\RR^d$, then $\Nrv(\calS)\cong\bigcup_{i\in[n]}S_i$.
\end{theorem}

\begin{proof}[Proof of \cref{prop:technical.set}]

We prove the contrapositive statement of $\bigcup_{i\in[k]} S_i= \DD^k\implies \bigcap_{i\in[k]} S_i\neq \emptyset$ by strong induction on the dimensionality $k$.  For the base case of $k=2$, observe that $S_1\cup S_2 = \{x: x^\T v_{12}\leq B_{12}-1 \text{ or } x^\T v_{12}\geq 1- B_{21}\}$, so $S_1\cup S_2 =\DD^2$ if and only if $B_{12} -1 \geq 1-B_{21}$, in which case the point $(1-B_{12}/2,B_{12}/2)\in S_1\cap S_2$, thereby the intersection is nonempty.

For $k>2$, suppose $\bigcup_{i\in[k]} S_i=\DD^k$. Our goal is to show that for all $J\subset[k]$, $\bigcap_{j\in J}S_j \neq \emptyset$.  Recall that 
\begin{equation}
    S_{i}= \bigcap_{j\in[k],j\neq i} \cbr{x\in \DD^k: x^\T v_{ij} \leq B_{ij} -1},
\end{equation}
and we define for any $J\subset[k]$ and $i\in[k]$
\begin{align}
  S'_{J,i} \coloneqq  \bigcap_{j\in J, j\neq i} \cbr{x\in \DD^k: x^\T v_{ij} \leq B_{ij} -1},
\end{align}
(we will drop the subscript $J$ as the discussions below will focus on a single $J$). 

We first show that $\bigcap_{i\in J}S_i\neq \emptyset$ for any $J\subset [k]$ with $|J|\leq k-1$.  By assumption, $\DD^k = \bigcup_{i\in[k]} S_i \subset \bigcup_{i\in [k]}S'_i$, and we argue that $\bigcup_{i\in J}S'_i = \DD^k$.
Suppose not, then let $z\notin \bigcup_{i\in J}S'_i$, and consider the line
\begin{equation}
L\coloneqq \cbr*{z+a{\sum_{i\notin J, j\in J}} v_{ij}: a \in\RR  }.
\end{equation}
First, no part of this line is contained in $\bigcup_{i\in J} S'_i$, because it does not contain the point $z\in L$, and $L$ runs parallel to and hence never intercepts any of the halfspaces defining each $S'_i$ for $i\in J$: let $i,j\in J$, $i\neq j$, then
\begin{equation}
    v_{ij} ^\T \sum_{n\notin J, m\in J}v_{nm} = \sum_{n\notin J, m\in J} (e_j^\T e_m - e_j^\T e_n -  e_i^\T e_m + e_i^\T e_n)= 1 - 0 - 1 + 0 = 0.
\end{equation}
Second, this line is partially not contained any $S'_i = \bigcap_{j\in J}\{x\in \DD^k : x^\T v_{ij}\leq B_{ij} -1 \}$ for $i\notin J$: let $i\notin J$ and $j\in J$, then
\begin{equation}
    v_{ij} ^\T \sum_{n\notin J, m\in J}v_{nm} = \sum_{n\notin J, m\in J} (e_j^\T e_m - e_j^\T e_n -  e_i^\T e_m + e_i^\T e_n)= 1 - 0 - 0 + 1 = 2;
\end{equation}
so points on $L$ with sufficiently large $a$'s are not contained in $\bigcup_{i\notin J} S'_i$, contradicting the assumption that $\bigcup_{i\in [k]}S'_i = \DD^k$.

Back to proving that  $\bigcap_{i\in J}S_i\neq \emptyset$ for any $J\subset [k]$ with $|J|\leq k-1$.  Since $\bigcup_{i\in J}S'_i = \DD^k$, by applying the inductive hypothesis to a $|J|$-dimensional instance derived from $\{S'_i:i\in J\}$ by removing the axes $\{e_i:i\notin J\}$, we get $\exists z'\in\bigcap_{i\in J} S'_i$.  Using similar arguments above, it can be shown that the line $L'\coloneqq \{z'+a\sum_{i\notin J, j\in J}v_{ij}: a \in\RR\}$ is entirely contained in $\bigcap_{i\in J} S'_i$ and partially in $\bigcap_{i \in J} S_i = \rbr{\bigcap_{i\in J}S'_i} \cap \rbr{\bigcap_{i\notin J, j\in J}\{x\in \DD^k : x^\T v_{ij}\leq B_{ij} -1 \}}$, so the intersection is nonempty.

We have thus established that any intersection of the strict subset of $\{S_i\}_{i\in[k]}$ is nonempty, and we will conclude with the Nerve theorem.
We have $\forall J\subset [k]$, $1\leq |J|\leq k-1$, $J\in\Nrv(\{S_1,\cdots, S_k\})$. Because we assumed in the beginning that $\bigcup_{i\in[k]} S_i=\DD^k$, it must follow that $[k]\in\Nrv(\{S_1,\cdots, S_k\})$ as well. Otherwise, the nerve is a $(k-1)$-dimensional simplex (each $n$-face is represented by its $n-1$ vertices) without its interior (represented by $[k]$), whose homology differs from that of $\DD^k$, then $\bigcup_{i\in[k]} S_i \cong \Nrv(\{S_1,\cdots, S_k\})\not\cong \DD^k$ by \cref{thm:nerve}, which contradicts our assumption that $\bigcup_{i\in[k]} S_i= \DD^k$.  Hence the nerve contains $[k]$, meaning $\bigcap_{i \in [k]} S_i\neq\emptyset$.
\end{proof}

\begin{proof}[Proof of \cref{prop:ot.intersect}]
  Recall the definitions of the objects $B_{ij}$, $C_i$ and $A_i$ in \cref{eq:decision.region,eq:geometry.A} of $\gamma^*$.
  Suppose $\bigcap_{i\in[k]} A_i = \emptyset$, then $\exists z\in \bigcap_{i\in[k]} (\DD^k\setminus A_i)$ by \cref{prop:technical.set}.  It then follows by definition that $\forall i\in[k]$, $\exists j\neq i$ s.t.~$z^\T v_{ij} < B_{ij}-1$.  Let $u:[k]\rightarrow[k]$ denote a mapping s.t.~the pairs $(i,u(i))$ satisfy this relation for all $i\in[k]$; note that there exists a nonempty $J\subseteq[m]$ s.t.~the undirected edges $\{(i, u(i)) : i\in J\}$ form a cycle because $u(i)\neq i$.  Also, there exist $m>0$ and measurable sets $F_i\subset \{x:x^\T v_{iu(i)} > z^\T v_{iu(i)}   \}\subset C_i$ s.t.~$\gamma^*(F_i,e_i)\coloneqq m_i \geq m$.

\begin{figure}[t]
    \centering
    \includegraphics[width=\linewidth]{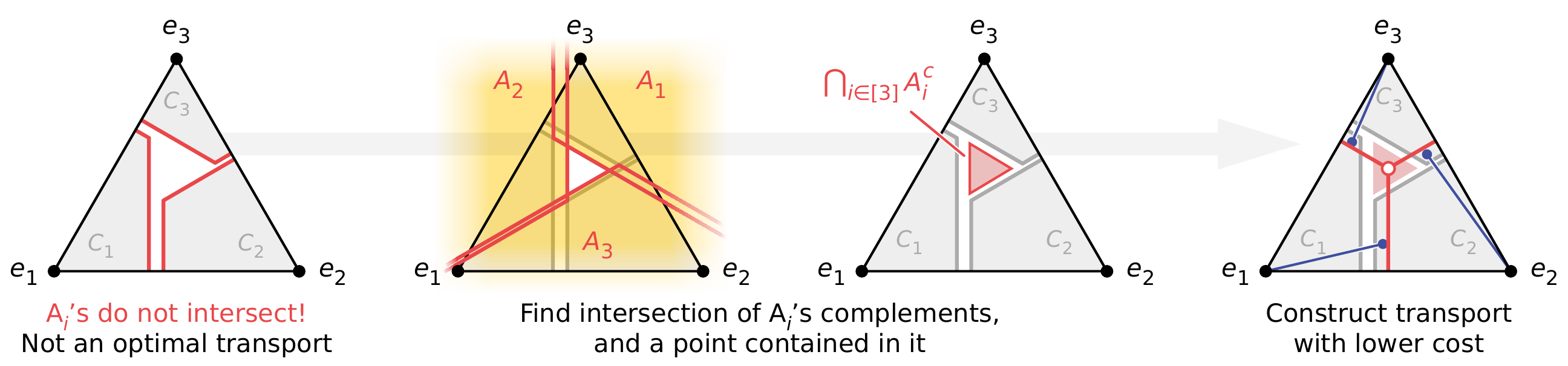}
    \caption{Illustration of the construction in the proof of \cref{prop:ot.intersect} for $k=3$.}
    \label{fig:geo.ot}
\end{figure}

  We show that the coupling $\gamma'\in\Gamma(p,q)$ given by
  \begin{equation}
    \gamma'(B, e_i) =\begin{cases}
   \gamma^*(B, e_i) & \text{if $i\notin J$,}\\ 
\begin{aligned}
\MoveEqLeft  \gamma^*(B \cap (\Delta_k\setminus F_i) ,e_i) \\ &+\frac{m_i-m}{m_i }\gamma^*(B \cap F_i ,e_i) + \frac{m}{m_{u^{-1}(i)}} \gamma^*(B\cap F_{u^{-1}(i)},e_{u^{-1}(i)})
\end{aligned}
  & \text{else}
\end{cases}
  \end{equation}
  has a lower transportation cost than $\gamma^*$ (see \cref{fig:geo.ot} for an illustration):
  \begin{align}
\int_{\Delta_k\times\calY} \|x-y\|_1\dif (\gamma^*-\gamma')(x,y) 
&= \sum_{i\in J}\frac{m}{m_i}\int_{F_i} \rbr*{\|x-e_i\|_1 - \|x-e_{u(i)}\|_1}\dif \gamma^*(x,e_i) \\
&= \sum_{i\in J}\frac{2m}{m_i}\int_{F_i} x^\T v_{iu(i)}\dif \gamma^*(x,e_i) \\
&> \sum_{i\in J}\frac{2m}{m_i}\int_{F_i} z^\T v_{iu(i)}\dif \gamma^*(x,e_i) \\
&= 2m\sum_{i\in J} z^\T v_{iu(i)} \\
&= 2m\sum_{i\in J} (z_{u(i)}-z_i) = 0,
  \end{align}
where line 2 follows from \cref{eq:simplex.dist}.
\end{proof}

Finally, we consider the complexity of the function class $\calG_k$ defined in \cref{eq:g_k}.  First, recall the definition of multi-class shattering, based on which the Natarajan dimension is defined~\citep[][Definitions 29.1 and 29.2]{shalev-shwartz2014UnderstandingMachineLearning}:

\begin{definition}[Multi-Class Shattering]
  Let $\calH$ be a class of functions from $\calX$ to $\{1,\cdots,k\}$.
  A set $S\coloneqq\{x_1,\cdots,x_n\}\subseteq\calX$ is said to be multi-class shattered by $\calH$ if $\exists f_0,f_1:S\rightarrow\{1,\cdots,k\}$ labelings satisfying $f_0(x_i)\neq f_1(x_i)$ for all $i\in[n]$, such that $\forall S_0,S_1$ that partition $S$ (i.e., $S_0\sqcup S_1= S$), $\exists h\in\calH$, s.t.\
  \begin{alignat}{2}
    h(x)&=f_0(x), &&\quad \forall x\in S_0, \quad\text{and} \\
    h(x)&=f_1(x), &&\quad \forall x\in S_1.
  \end{alignat}
\end{definition}

\begin{definition}[Natarajan Dimension] \label{def:nat}
    Let $\calH$ be a class of functions from $\calX$ to $\{1,\cdots,k\}$. The Natarajan dimension of $\calH$, denoted by $\dN(\calH)$, is the size of the largest subset of $\calX$ multi-class shattered by $\calH$.
\end{definition}

\begin{proof}[Proof of \cref{thm:natdim}]
We associate $e_i$ with the label $i$, $\forall i\in\{1,\cdots,k\}$. 
We first show that $\dN(\calG_k)\geq k-1$ by constructing a set of cardinality $k-1$ that is shattered by $\calG_k$, then show that $\dN(\calG_k)<k$ by contradiction.

\paragraph{Lower Bound.}

Consider the set $S=\{e_1,e_2,\cdots,e_{k-1}\}$ and let $f_0(e_j)=j$ and $f_1(e_j)=k$ for all $j\in[k-1]$, which satisfy $f_0\neq f_1$ on all $x\in S$.  Let $S_0\sqcup S_1=S$ be arbitrary, and define
\begin{equation}
\iota(j)\coloneqq\begin{cases}
  \1[e_j\in S_1] & \text{if $j\in[k-1]$} \\
  0   & \text{if $j=k$.} 
\end{cases} 
\end{equation}
Consider $g_z\in\calG_k$ with parameters
\begin{equation}
  z = \frac1k\, \boldsymbol{1}_k - \sum_{j=1}^{k-1} \iota(j) \sum_{\ell\neq j} v_{j\ell},
\end{equation}
where boldface $\boldsymbol{1}_k\in\RR^k$ denotes the vector of all ones.  Observe that
\begin{align}
  z^\T v_{nm} 
  &= - \sum_{j=1}^{k-1} \iota(j) \sum_{\ell\neq j} v_{j\ell} ^\T v_{nm}\\
  &= - \sum_{j=1}^{k-1} \iota(j) \sum_{\ell\neq j} (e_\ell^\T e_m - e_\ell^\T e_n-  e_j^\T e_m + e_j^\T e_n)\\
  &= \sum_{j=1}^{k-1} \iota(j)\, (\1[n\neq j] - \1[m\neq j]) + (k-1)\sum_{j=1}^{k-1} \iota(j) \, (e_j^\T e_m-  e_j^\T e_n) \\
  &= (k-1)(\iota(m) - \iota(n)) +   ( \iota (m)-\iota(n)  )\\
  &= k(\iota(m) - \iota(n)).
\end{align}

Recall from \cref{eq:g_k.alt} that for all $i,n\in[k]$,
  \begin{equation}
    g_z(e_i) = e_n \text{ is eligible}\iff  e_i^\T v_{nm} \leq z^\T v_{nm},\quad\forall m\neq n;
  \end{equation}
so in our case, it follows that for all $i\in[k-1]$ and $j\neq i$,
\begin{equation}
      g_z(e_i) = e_i \text{ is eligible}
      \iff  -1 \leq k(\iota(m) - \iota(i)),\quad\forall m\neq i,
\end{equation}
and
\begin{equation}
  g_z(e_i) = e_j \text{ is eligible} \iff  1 \leq k(\iota(i)-\iota(j))\quad\text{and}\quad 0 \leq k(\iota(m)-\iota(j)), \quad\forall m\neq i
\end{equation}
(also, recall $\iota(k)\coloneqq0$).

Observe that for any $i\in[k-1]$, if $\iota(i)=0$, then $g_z(e_i)=e_j$ is ineligible for all $j\neq i$, then we must have $g_z(e_i)=e_i$.  Otherwise, if $\iota(i)=1$, then $e_k$ is always eligible, so $g_z(e_i)=e_k$ due to the tie-breaking rule.  Therefore, $g_z$ is a witness function, and we conclude that $\calG_k$ shatters $S$.

\paragraph{Upper Bound.}
  Let $S=(x_1,\cdots,x_k)$ be given, along with $f_0,f_1:S\rightarrow[k]$ satisfying $f_0(x)\neq f_1(x)$ for all $x\in S$.    Suppose $\calG_k$ shatters $S$.  Let $g_{z}\in\calG_k$ denote a witness function for the partitioning of $S_0=S$ and $S_1=\emptyset$, and $g_{z'}\in\calG_k$ that for the partitioning of $S'_0=\emptyset$ and $S'_1=S$.

We will reuse an argument from an earlier proof.  
  Denote the decision boundaries of $g_z$ (analogously for $g_{z'}$) by $B_{ij}\coloneqq z^\T v_{ij}+1$, and the decision regions by $C_{i}\coloneqq \bigcap_{j\neq i} \cbr{x\in\Delta_k: x^\T v_{ij} \leq B_{ij} -1}$.  
  By definition of $\calG_k$, $x\in C_i\implies g_z(x)=e_i$ is eligible.
  Then, define the difference in the boundaries between $g_z$ and $g_{z'}$ by $d_{ij}\coloneqq B'_{ij} - B_{ij}$.
Construct a directed graph of $k$ nodes where $(i,j)$ is an edge iff $d_{ij}>0$,  which is acyclic as shown in the proof of \cref{lem:monge.exist}.

First, consider the case where $\exists x,x'\in S$ s.t.~$\{f_0(x),f_1(x)\}=\{f_0(x'),f_1(x')\}$.  W.l.o.g., assume $i\coloneqq f_0(x_1)=f_1(x_2)$ and $j\coloneqq f_1(x_2)=f_0(x_1)$, then we have 
\begin{alignat}{4}
   g_{z}(x_1)&=e_{i}, \;\; &g_{z'}(x_1)&=e_{j} &&\implies d_{ji}>0, \\
   g_{z}(x_2)&=e_{j}, &g_{z'}(x_2)&=e_{i} &&\implies d_{ij}>0
\end{alignat}
(after taking into account of the tie-breaking rule), however, this would imply a cycle in the graph, which is a contradiction.

Next, if $\{f_0(x),f_1(x)\}$ differs for all $x\in S$, then we may assume w.l.o.g.~$f_0(x_i)=e_i$ and $f_1(x_i)=e_{i+1}$ for all $i\in[k]$ (where the index of $k+1$ means $1$).
Then
\begin{equation}
   g_{z}(x_i)=e_{i},\;\; g_{z'}(x_i)=e_{i+1} \implies d_{i+1,i}>0,\quad\forall i\in[k];
\end{equation}
again, this would imply a cycle in the graph, hence a contradiction.  Therefore, we conclude that $\calG_k$ cannot shatter any $S\subset\Delta_k$ of cardinality $k$.
\end{proof}

In addition, on data distributions that satisfy $X=Y$, $X\in\Delta_k$, applying the $\ell_1$ loss of $\ell(\hat y,y)\coloneqq \|\hat y-y\|_1$ to $\calG_k$ yields a function class with pseudo-dimension of $k-1$:

\begin{definition}[Pseudo-Shattering]
  Let $\calF$ be a class of functions from $\calX$ to $\RR$. A set $\{x_1,\cdots,x_n\}\subseteq\calX$ is said to be pseudo-shattered by $\calF$ if $\exists t_1,\cdots,t_n\in\RR$ threshold values s.t.\ $\forall b_1,\cdots,b_n\in\{0,1\}$ binary labels, $\exists f\in\calF$ s.t.\ $\1[f(x_i)\geq t_i] = b_i$ for all $i\in[n]$.
\end{definition}

\begin{definition}[Pseudo-Dimension] \label{def:pdim}
    Let $\calF$ be a class of functions from $\calX$ to $\RR$. The pseudo-dimension of $\calF$, denoted by $\dP(\calF)$, is the size of the largest subset of $\calX$ pseudo-shattered by $\calF$.
\end{definition}

\begin{theorem}\label{cor:pdim}
  Define $\calF_k \coloneqq \cbr*{ x\mapsto \enVert*{ g(x)-x }_1 : g\in\calG_k }$. We have $\dP(\calF_{k})= k-1$.
\end{theorem}

\begin{proof}
The proof shares the same arguments as that of \cref{thm:natdim}.  We will only show the upper bound, and remark that the lower bound can be established using a similar construction of that in \cref{thm:natdim}.

Let $x_1,\cdots,x_k$ be given, and suppose there exists thresholds $r_1,\cdots,r_k$ s.t.~$\calF_k$ shatters the set of points.  It follows that $\exists g_{z},g_{z'}\in\calG_k$ s.t.~$\|g_{z'}(x_i)-x_i\|_1<r_i\leq \|g_z(x_i)-x_i\|_1$ for all $i$, which means that $g_z(x_i)\neq g_{z'}(x_i)$. But by the arguments in the proof of the upper bound of \cref{thm:natdim}, such $(g_{z},g_{z'})$ pair does not exists, contradicting the shattering assumption.
\end{proof}

Finally, for all $i\in\{1,\cdots,k\}$, define restriction of $\calG_k$ to class $i$ by
\begin{equation}
  \calG_{k,i}\coloneqq\{x\mapsto \1[g(x)=e_i]:g\in\calG_k\} \subset \{0,1\}^{\Delta_k}. \label{eq:bin.g}
\end{equation}
Because it is a binary function, its \textit{VC dimension}, denoted by $\dVC(\calG_{k,i})$, is equivalent to its Natarajan dimension by \cref{def:nat}; moreover, because $\calG_{k,i}$ is derived from $\calG_k$ by an output remapping, its Natarajan dimension is clearly upper bounded by that of the latter:

\begin{corollary}\label{cor:vcdim}
  $\dVC(\calG_{k,i})\leq \dN(\calG_k)=k-1$. 
\end{corollary}

\section{Experiment Details}\label{sec:exp.additional}

\subsection{Datasets and Tasks}

\paragraph{Adult \textnormal{\citep{kohavi1996ScalingAccuracyNaiveBayes}}.}  We consider the binary classification task of whether the annual income of an individual is over or below \$50k per year ($|\calY|=2$) given attributes including gender, race, age, education level, etc. The data are collected from the 1994 US Census.  We let gender be the sensitive attribute ($|\calA|=2$).  It contains 48,842 examples in total, which we split for pre-training/post-processing/testing by 0.35/0.35/0.3.

\paragraph{ACSIncome \textnormal{\citep{ding2021RetiringAdultNew}}.} It is an extension of the Adult dataset with data collected from the US Census Bureau.  We consider income prediction under two settings.  In the binary setting, the task is to predict whether the annual income of an individual is over or below \$50k per year ($|\calY|=2$), with gender as the sensitive attribute ($|\calA|=2$).  In the multi-group multi-class setting, we create five income buckets of \mbox{$<$15000}, \mbox{[15000,30000)}, \mbox{[30000,48600)}, \mbox{[48600,78030)}, \mbox{$\geq$78030}, and group the data into five race categories of ``American Indian or Alaska Native alone'', ``Asian'', ``Native Hawaiian or Other Pacific Islander alone'', ``Black or African American alone'', ``Other'', and ``White alone'' (same as in Adult). It contains 1,664,500 examples, which we split for pre-training/post-processing/testing by 0.63/0.07/0.3.
  
\paragraph{BiasBios \textnormal{\citep{de-arteaga2019BiasBiosCase}}.}

The task is to determine the occupation ($|\calY|=28$) of female and male individuals ($|\calA|=2$) by their raw text biographies.  The data are mined from the Common Crawl corpus.  In this dataset, gender is the sensitive attribute, and is observed to correlate with certain occupations such as software engineer and nurse.
We use the version of BiasBios scrapped and hosted by \citet{ravfogel2020NullItOut} with 393,423 examples in total, which we split for pre-training/post-processing/testing by 0.35/0.35/0.3.

This experiment is of particular interest because of the increasing popularity of large language models and the fairness concerns regarding their usage.  In particular, the uncurated corpora (e.g., crawled from the internet) on which the language models are pre-trained may contain historical social bias, and empirical investigations have shown that such bias could be propagated and amplified in downstream applications~\citep{bolukbasi2016ManComputerProgrammer,zhao2018GenderBiasCoreference,abid2021PersistentAntiMuslimBias}.

\paragraph{Communities \& Crime \textnormal{\citep{redmond2002DatadrivenSoftwareTool}}.}

The Communities \& Crime tabular dataset contains the socioeconomic and crime data of communities in 46 US states, and the task is to predict the number of violent crimes per 100k population given attributes ranging from the racial composition of the community, their income and background, and law enforcement resource. The data come from the 1990 US Census, 1990 LEMAS survey, and 1995 FBI Uniform Crime Reporting program.  We bin the rate of violent crime into five classes ($|\calY|=5$), and we treat race as the sensitive attribute by the presence of minorities ($|\calA|=4$): a community does not have a significant presence of minorities if White makes up more than 95\% of the population, otherwise the largest minority group is considered to have a significant presence (Asian, Black, or Hispanic).  It contains 1,994 examples, which we split for pre-training/post-processing/testing by 0.35/0.35/0.3.

\subsection{Additional Details}\label{sec:hyper}

On each task, we first create the pre-training split from the dataset and train a linear logistic regression scoring model using the implementation provided in \texttt{scikit-learn}~\citep{pedregosa2011ScikitlearnMachineLearning}.  Then, we randomly split the remaining data for post-processing and testing with 10 different seeds and aggregate the results as presented in \cref{fig:exp,fig:exp.2} (the pre-trained model remains the same).

The linear programs of our \cref{alg:cont} are implemented using the interface of \texttt{cvxpy}~\citep{diamond2016CVXPYPythonEmbeddedModeling}, and are solved using the COIN-OR \texttt{Cbc} solver that is based on the \textit{branch and cut} method~\citep{john_forrest_2023_7843975}.
Finally, the BERT model in BiasBios experiments is loaded through the Hugging~Face \texttt{Transformers} library~\citep{wolf2020TransformersStateoftheArtNatural}.

\paragraph{Hyperparameters.}

The tradeoff curves in \cref{fig:exp,exp:2} are generated with the following fairness tolerance/strictness settings.

For our method, $\alpha$ is set to:
\begin{itemize}
  \item ACSIncome (binary).~~0.2, 0.18, 0.16, 0.14, 0.12, 0.1, 0.08, 0.06, 0.04, 0.02, 0.01, 0.005, 0.001, 0.0.
  \item ACSIncome ($5$-group, $5$-class).~~0.32, 0.3, 0.28, 0.26, 0.24, 0.22, 0.2, 0.18, 0.16, 0.14, 0.12, 0.1, 0.08, 0.06, 0.04, 0.02, 0.01, 0.0.
  \item BiasBios.~~0.08, 0.07, 0.06, 0.05, 0.04, 0.03, 0.02, 0.01, 0.008, 0.006, 0.004, 0.002, 0.001, 0.0.
  \item Adult.~~0.16, 0.14, 0.12, 0.1, 0.08, 0.06, 0.04, 0.02, 0.01, 0.008, 0.006, 0.004, 0.002, 0.001, 0.0.
  \item Communities.~~0.6, 0.55, 0.5, 0.45, 0.4, 0.35, 0.3, 0.25, 0.2, 0.15, 0.1, 0.05, 0.01, 0.0.
\end{itemize}

For \texttt{FairProjection}, we use the default settings that came with the code/package; in particular, increasing the number of iterations to over 1,000 did not improve performance.  The tolerance is set to:
\begin{itemize}
  \item ACSIncome (binary).~~0.3, 0.2, 0.18, 0.16, 0.14, 0.12, 0.1, 0.08, 0.06, 0.04, 0.02, 0.01, 0.005, 0.001, 0.0.
  \item ACSIncome ($5$-group, $5$-class).~~0.5, 0.4, 0.35, 0.3, 0.25, 0.2, 0.15, 0.1, 0.08, 0.06, 0.04, 0.02, 0.01, 0.0.
  \item BiasBios.~~1.0, 0.9, 0.8, 0.7, 0.6, 0.5, 0.4, 0.3, 0.2, 0.1, 0.05, 0.0.
  \item Adult.~~0.6, 0.5, 0.4, 0.3, 0.2, 0.1, 0.0.
  \item Communities.~~1.0, 0.8, 0.6, 0.4, 0.3, 0.2, 0.1, 0.0.
\end{itemize}

\subsection{Finding Feasible Point on Line~\ref{lst:ln.7}}\label{sec:find.point}

\Cref{lst:ln.7} of \cref{alg:cont} involves finding a feasible point in the intersection of halfspaces, which can be obtained with the following linear program:
\begin{equation}
  \!\min_{{z\in\RR^k}} 0 \quad\text{s.t.}\quad 
  z^\T v_{ij} \leq B_{ij}-1, \quad  \forall i,j\in[k],\, i\neq j.
\end{equation}

As illustrated in \cref{fig:transport}, the point $z$ that is returned determines the center of the boundaries of the extracted transport maps $\calT_a\in\calG_k$.  Because of the machine learning folklore that classifiers with larger margin enjoy better generalization properties, we instead use the follow quadratic program (of a least-squares problem) that maximizes the margins in our experiments for point-finding:
\begin{equation}
  \!\min_{{z\in\RR^k}} \sum_{i\neq j} \frac{\|z^\T v_{ij} - (B_{ij}-1)\|_2^2}{(2-B_{ij}-B_{ji})^2} \quad\text{s.t.}\quad 
  z^\T v_{ij} \leq B_{ij}-1, \quad  \forall i,j\in[k],\, i\neq j.
\end{equation}
Our preliminary experiments showed that using the quadratic program for point-finding led to better post-processing performance with \cref{alg:cont} than using the linear program, both in terms of the error rate and $\DPGap$ during inference.

\begin{table}[tp]
\caption{Running time (in seconds) of post-processing algorithms under the strictest tolerance setting (see \cref{sec:hyper}), averaged over three random splits.  Our algorithm is run on a single core of an Intel Xeon Silver 4314 CPU, and \texttt{FairProjection} is run on an NVIDIA RTX A6000 GPU.}
\label{tab:runtime}
\centering
    \scalebox{0.9}{%
\begin{tabular}{p{5.5cm}C{1.5cm}C{1.5cm}C{1.5cm}C{1.5cm}C{2.5cm}}
\toprule
 & \multicolumn{2}{c}{ACSIncome} & BiasBios & Adult & Communities \\
\midrule
Groups  & 2 & 5 & 2 & 2 & 4 \\
Classes  & 2 & 5 & 28 & 2 & 5 \\
Examples (post-processing split) & \multicolumn{2}{c}{116,515} &  137,698 & 17095 & 698 \\
\midrule
Ours (CPU)          & 2.56 &  109 & 797 & 0.43 & 0.38 \\
\texttt{FairProjection-KL} (GPU) & 33 & 38 & 99 & 15 & 13 \\
\bottomrule
\end{tabular}
}%
\end{table}

\begin{figure}[p]
    \centering
    \includegraphics[height=.345\linewidth]{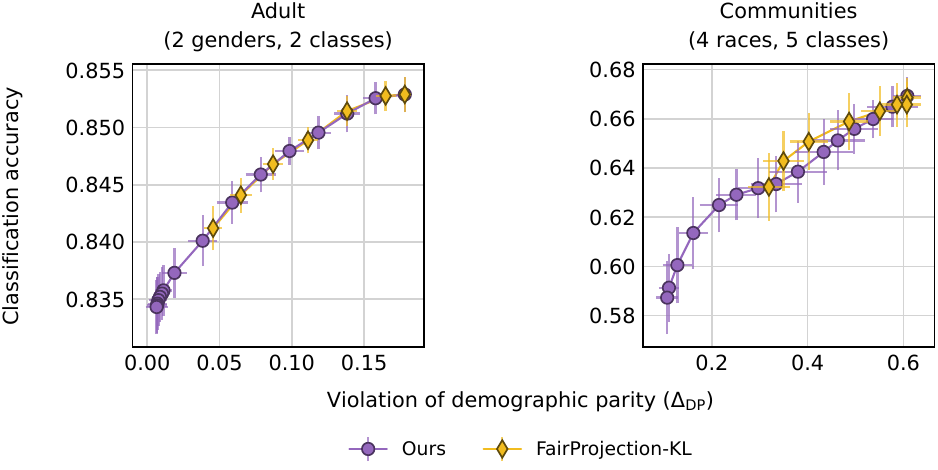}
  \caption{Tradeoff curves between accuracy and $\DPGap$ (\cref{def:dp}).  Scoring model is logistic regression. Error bars indicate the standard deviation over 10 runs with different random splits. Running time is reported in \cref{tab:runtime}.  On both datasets, \texttt{FairProjection-KL} and \texttt{FairProjection-CE} have similar results.}
    \label{fig:exp.2}
\end{figure}

\end{document}